\documentclass[nohyperref]{article}

\usepackage{microtype}
\usepackage{graphicx}
\usepackage{subcaption}
\usepackage{booktabs} 

\usepackage{hyperref}

\usepackage[accepted]{icml2022}

\usepackage{amsmath}
\usepackage{amssymb}
\usepackage{mathtools}
\usepackage{amsthm}

\usepackage[capitalize,noabbrev]{cleveref}

\theoremstyle{plain}
\newtheorem{theorem}{Theorem}[section]
\newtheorem{proposition}[theorem]{Proposition}
\newtheorem{lemma}[theorem]{Lemma}
\newtheorem{corollary}[theorem]{Corollary}
\theoremstyle{definition}
\newtheorem{definition}[theorem]{Definition}

\theoremstyle{remark}

\usepackage{nicefrac}
\usepackage{dsfont}
\usepackage[inline]{enumitem}
\usepackage{pifont}
\newcommand{\cB}{\mathcal{B}}

\newcommand{\cH}{\mathcal{H}}
\newcommand{\cN}{\mathcal{N}}
\newcommand{\cO}{\mathcal{O}}
\newcommand{\cP}{\mathcal{P}}
\newcommand{\cR}{\mathcal{R}}

\newcommand{\cU}{\mathcal{U}}
\newcommand{\cX}{\mathcal{X}}
\newcommand{\cY}{\mathcal{Y}}
\newcommand{\cZ}{\mathcal{Z}}
\newcommand{\C}{\mathbb{C}}
\newcommand{\R}{\mathbb{R}}

\newcommand{\bE}{\mathbb{E}}
\newcommand{\bV}{\mathbb{V}}

\newcommand{\bP}{\mathbb{P}}
\newcommand{\Id}{\mathds{1}}

\newcommand{\pvec}{\mathbf{p}}
\newcommand{\evec}{\mathbf{e}}
\newcommand{\AUC}{\mathrm{AUC}}
\newcommand{\phat}{\hat{p}}
\newcommand{\abs}[1]{\lvert#1\rvert}

\newcommand{\norm}[2]{\left\|#1\right\|_{#2}}
\makeatletter
\def\norm{\@ifnextchar[{\@with}{\@without}}
\def\@with[#1]#2{\left\|#2\right\|_{#1}}
\def\@without#1{\left\|#1\right\|}
\makeatother

\newcommand{\secref}[1]{Sec.~\ref{#1}} 

\icmltitlerunning{Certifying Out-of-Domain Generalization for Blackbox Functions}

\begin{document}

\twocolumn[
\icmltitle{Certifying Out-of-Domain Generalization for Blackbox Functions}



\icmlsetsymbol{equal}{*}

\begin{icmlauthorlist}
\icmlauthor{Maurice Weber}{eth}
\icmlauthor{Linyi Li}{uiuc}
\icmlauthor{Boxin Wang}{uiuc}
\icmlauthor{Zhikuan Zhao}{eth}
\icmlauthor{Bo Li}{uiuc}
\icmlauthor{Ce Zhang}{eth}
\end{icmlauthorlist}

\icmlaffiliation{eth}{Department of Computer Science, ETH Zurich}
\icmlaffiliation{uiuc}{UIUC, USA}

\icmlcorrespondingauthor{Maurice Weber}{maurice.weber@inf.ethz.ch}
\icmlcorrespondingauthor{Ce Zhang}{ce.zhang@inf.ethz.ch}

\icmlkeywords{Machine Learning, ICML}

\vskip 0.3in
]


\printAffiliationsAndNotice{}  

\begin{abstract}
Certifying the robustness of
model performance under bounded
data distribution drifts has recently
attracted intensive interest under
the umbrella of \textit{distributional robustness}.
However, existing techniques either make strong assumptions on the model class and loss functions that can be certified, such as smoothness expressed via Lipschitz continuity of gradients, or require to solve complex optimization problems.
As a result, the wider application of these
techniques is currently
limited by its \textit{scalability} and
\textit{flexibility} --- these techniques often do not scale to large-scale datasets with modern deep neural networks or cannot handle loss functions which may be non-smooth such as the 0-1 loss.
In this paper, we focus on the problem of
certifying distributional robustness
for \textit{blackbox models and
bounded loss functions}, and propose a novel certification framework based on the Hellinger distance.
Our certification technique scales to ImageNet-scale datasets, complex models,
and a diverse set of loss functions.
We then focus on one specific application
enabled by such scalability and flexibility, i.e.,
certifying out-of-domain generalization
for large neural networks and
loss functions such as accuracy and
AUC.
We experimentally validate our certification method on a number of datasets, ranging from ImageNet, where we provide the first non-vacuous certified out-of-domain generalization,
to smaller classification tasks where we are able to compare with the state-of-the-art and show that our method performs considerably better.
\end{abstract}

\vspace{-0.75em}
\section{Introduction}
\label{sec:introduction}
\vspace{-0.25em}
The wide application of machine
learning models in the real world
brings an emerging
challenge of understanding the
performance of a machine learning
model under different data
distributions ---
ML systems operating autonomous vehicles which are trained based on data collected in the northern hemisphere might fail when deployed in desert-like environments or under different weather conditions~\cite{volk2019towards,dai2018dark},
while recognition systems have been shown to fail when deployed in new environments~\cite{beery2018recognition}.
Similar concerns also apply to many  mission-critical applications
such as medicine and cyber-security~\cite{koh2021wilds,albadawy2018deep,gulrajani2021in}.
In all these applications,
it is imperative to have a sound understanding of the model's robustness and possible failure cases in the presence of a shift in the data distribution, and to have corresponding guarantees on the performance.

Recently, this problem has attracted
intensive interest under the umbrella of \textit{distributional robustness}~\cite{scarf1958min,ben2013robust,gao2016distributionally,kuhn2019wasserstein,blanchet2019quantifying,duchi2021statistics}. Specifically,
let $P$ be a joint data distribution over
features $X \in \mathcal{X}$ and labels $Y \in \mathcal{Y}$,
and let $h_\theta: \mathcal{X} \to \mathcal{Y}$ be a
machine learning model parameterized by $\theta$. For a loss function
$\ell: \mathcal{Y} \times \mathcal{Y} \to \mathbb{R}$, we hope to compute
\begin{equation}
    \label{eq:worst-case-risk}
    \cR_\theta(\cU_P) := \sup_{Q \in \cU_P}\bE_{(X, Y)\sim Q}[\ell(h_\theta(X), Y)]
\end{equation}
where $\cU_P \subseteq \cP(\cZ)$ is a set of probability distributions on $\cZ$, called the \textit{uncertainty set}.
Intuitively, this
measures the \textit{worst-case risk}
of $h_\theta$ when the data
distribution drifts from $P$
to another distribution
in $\cU_P$.

Providing a technical solution to
this problem has gained increased attention
over the years, as summarized in Table~\ref{tab:landscape}.
However,
most --- if not all --- existing approaches,
impose strong constraints such as
bounded Lipschitz gradients on
both $h$ and $\ell$ and rely on
expensive
certification methods such as
direct minimax optimization. As a result,
these methods have been applied only to small-scale datasets
and ML models.

\begin{table*}[t!]
\small
\centering
\begin{tabular}{c c c c c c c}
\toprule
Ref. & Assumptions on $\ell$ & Assumption on $h$ & Distance & Largest Dataset \\
\midrule
~\cite{gao2016distributionally} & \multicolumn{2}{c}{Generalised Lipschitz Continuity} & Wasserstein & -- \\
~\cite{sinha2017certifying} & Bounded, Smoothness & Smoothness & Wasserstein & MNIST\\
~\cite{staib2019distributionally} & Bounded, Continuous & Kernel Methods & MMD & --\\
~\cite{shafieezadeh2019regularization} & \multicolumn{2}{c}{Lipschitz Continuity} & Wasserstein & -- \\
~\cite{blanchet2019quantifying} & Bounded, Smoothness & Smoothness & Wasserstein & -- \\
~\cite{cranko2021generalised} & \multicolumn{2}{c}{Generalised Lipschitz Continuity} & Wasserstein & -- \\
\midrule
Our Method & Bounded & \textit{any} Blackbox & Hellinger & ImageNet \\
\bottomrule
\end{tabular}
\caption{Current landscape of certified distributional robustness.}
\vspace{-1em}
\label{tab:landscape}
\end{table*}

In this paper, we consider the case
that both \textit{$h$ and $\ell$ can be non-convex and non-smooth} ---
$h$ can be a full-fledged neural network,
e.g., ImageNet-scale EfficientNet-B7~\cite{tan2019efficientnet}, and $\ell$
can be a general non-smooth loss function such as
the 0-1 loss. We provide, to our best knowledge,
the first practical method
for blackbox functions that scales to real-world, ImageNet-scale neural networks
and datasets. Our key
innovation is a novel algorithmic
framework that arises from bounding inner
products between elements of a suitable Hilbert space.
Specifically, we can
characterize the upper bound
of the performance of $h$ on any $Q$ within the uncertainty set
as a function of the Hellinger distance,
a specific type of $f$-divergence,
and the expectation and variance
of the loss of $h$ on $P$.

We then apply our
framework to the
problem of
certifying the out-of-domain generalization performance
of a given classifier,
taking advantage of its
scalability and flexibility.
Specifically, let $P$ be the
\textit{in-domain distribution},
and $h_\theta$ a classifier.
Then, to reason about the performance of $h_\theta$ on shifted distributions $Q$, we provide a certificate
in the following form:
\begin{equation}
    \begin{aligned}
        \forall Q\colon \, &\texttt{dist}(Q, P) \leq \rho \\
        &\implies \, \bE_{(X, Y)\sim Q}[\ell(h(X), Y)] \leq C_\ell(\rho,\,P)
    \end{aligned}
\end{equation}
where $C_\ell$ is a bound which depends on the distance $\rho$ and the distribution $P$.
This requires several nontrivial
instantiations of our framework with careful practical considerations.
To this end,
we first develop a certification algorithm
that relies only on a finite set of
samples from the in-domain distribution $P$.
Moreover, we also instantiate
it with different domain drifting models
such as label drifting and covariate drifting,
connecting the general Hellinger distance
to the degree of domain drifting specific
to these scenarios.
We then consider a diverse
range of loss functions,
including JSD loss, 0-1 loss,
and AUC.
To the best of our knowledge,
we provide the first certificate
for such diverse realistic scenarios, which is able to scale to large problems.

Last but not least, we conduct intensive
experiments verifying the efficiency
and effectiveness of our result.
Our method is able to scale to
datasets and neural networks as large as
ImageNet and full-fledged models like EfficientNet-B7
and BERT.
We further apply our method
on smaller-scale datasets, in order to compare with strong, state-of-the-art
methods. We show that our method provides
much tighter certificates.

Our contributions can be summarized
as follows:
\begin{itemize}
\vspace{-0.5em}
\item We present a novel framework which provides a non-vacuous, computationally tractable bound to the distributionally robust worst-case risk $\cR_\theta(\cU_P)$ for general bounded loss functions $\ell$ and models $h$.
\vspace{-0.5em}
\item We apply this framework to the problem of certifying out-of-domain generalization for blackbox functions and provide a means to certify distributional robustness in specific scenarios such as label and covariate drifts.
\vspace{-0.5em}
\item We provide an extensive experimental study of our approach on a wide range of datasets including the large scale ImageNet~\cite{russakovsky2015imagenet} dataset, as well as NLP datasets with complex models.
\end{itemize}

\section{Distributional Robustness
for Blackbox Functions}
\label{sec:gramian}
In this section, we present our main results, namely, a computationally tractable upper bound to the worst-case risk~\eqref{eq:worst-case-risk} for uncertainty sets expressed in terms of Hellinger balls around the data-generating distribution $P$.
The technique is based on the non-negativity of Gram matrices which, by expressing expectation values as inner products between elements of a suitable Hilbert space, can be leveraged to relate expectation values of a blackbox function under different probability distributions $P$ and $Q$.\footnote{The idea behind our methods
is inspired by how Gram matrices are used in quantum chemistry~\cite{weinhold1968lower,weber2021toward} to bound expectation values of quantum observables. However, the adaptation to machine learning is nontrivial and requires careful analysis.}
We describe the underlying technique leading to our main result in Theorem~\ref{thm:main}, which upper bounds the worst-case population loss using both the expectation and variance.

For the remainder of this section, to simplify notation and maintain generality, we consider generic loss functions $\ell\colon\cZ\to\R_+$ which
contain the model $h$ and take inputs from a generic input space $\cZ$. For example in the context of supervised learning, $\cZ = \cX \times \cY$ can be the product space of features and labels and the loss $\ell(z) = \tilde{\ell}(h_\theta(x),\,y)$ can be seen as a composition of the loss function $\Tilde{\ell}$ and the model $h_\theta$.
We denote the set of probability measures on the space $\cZ$ by $\cP(\cZ)$.
For two measures $\mu,\,\nu$ on $\cZ$, we say that $\nu$ is absolutely continuous with respect to $\mu$, denoted $\nu\ll\mu$ if $\mu(A) = 0$ implies that $\nu(A) = 0$ for any measurable set $A\subseteq \cZ$.
Among the plethora of distances between probability measures, such as total variation and Wasserstein distance, a particularly popular choice is the family of $f$-divergences which has been extensively studied in the context of distributionally robust optimization~\cite{ben2013robust,lam2016robust,duchi2019variance,duchi2021statistics}. In this paper, we focus on the Hellinger distance, which is a particular type of $f$ divergence.
\begin{definition}[Hellinger-distance]
    Let $P,\,Q\in\cP(\cZ)$ be probability measures on $\cZ$ that are absolutely continuous with respect to a reference measure $\mu$ with $P,\,Q\ll\mu$.
    The Hellinger distance between $P$ and $Q$ is defined as
    \begin{equation}
        H(P,\,Q) := \sqrt{\frac{1}{2}\int_{\cZ}\left(\sqrt{p(z)} - \sqrt{q(z)}\right)^2\,d\mu(z)}
    \end{equation}
    where $p = \frac{dP}{d\mu}$ and $q = \frac{dQ}{d\mu}$ are the Radon-Nikodym derivatives of $P$ and $Q$ with respect to $\mu$. The Hellinger distance is independent of the choice of the reference measure $\mu$.
\end{definition}
The Hellinger distance is bounded with values in $[0,\,1]$, with $H(P,\,Q) = 0$ if and only if $P = Q$ and the maximum value of $1$ attained when $P$ and $Q$ have disjoint support. Furthermore, $H$ defines a metric on the space of probability measures and hence satisfies the triangle inequality.
We will now show how the Hellinger distance can be expressed in terms of an inner product between elements of a suitable Hilbert space, which ultimately enables us to use the theory of Gram matrices to derive an upper bound on the worst-case population risk~\eqref{eq:worst-case-risk} for uncertainty sets given by Hellinger balls.
Consider the Hilbert space $L_2(\cZ,\,\Sigma,\,\mu)$\footnote{We take $\Sigma$ to be the Borel $\sigma$-algebra on $\cZ$, being the smallest $\sigma$-algebra containing all open sets on $\cZ$.} of square-integrable functions $f\colon\cZ \to \R$, endowed with the inner product $\langle f,\,g\rangle = \int_{\cZ} fg\,d\mu$.
Within this space, we can identify any probability distribution $P \ll \mu$ with a unit vector $\psi_P \in L_2(\cZ,\,\Sigma,\,\mu)$ via the square root of its Radon-Nikodym derivative $\psi_P := \sqrt{dP / d\mu}$.
This mapping enables us to write the Hellinger distance and, more generally, expectation values in terms of inner products. To see this, note that for any two probability measures $P,\,Q$ on $\cZ$, it holds that
\begin{equation}
    \langle \psi_P,\,\psi_Q\rangle = \int_\cZ \sqrt{dP}\sqrt{dQ} = 1 - H^2(P,\,Q)
\end{equation}
and similarly, for any essentially bounded function $f\in L_\infty$, we have
\begin{equation}
    \bE_{P}[f(Z)] = \int_{\cZ} f(z)\,dP(z) = \langle \psi_P,\,f\cdot\psi_P\rangle
\end{equation}
where the product $(f\cdot\psi_P)(z) = f(z)\cdot\psi_P(z)$ is to be understood as pointwise multiplication\footnote{More precisely, every $f\in L_\infty(\cZ,\,\Sigma,\,\mu)$ defines a bounded linear operator $M_f\colon L_2 \to L_2$ acting on elements of $L_2$ via pointwise multiplication, $g \mapsto M_f(g) := f \cdot g$ with $(f\cdot g)(z) = f(z)\cdot g(z)$ for any $z\in\cZ$.}.
For $f\in L_\infty$, consider the Gram matrix of the Hilbert space elements $\psi_Q$, $\psi_P$ and $f\cdot\psi_P$, defined as
\begin{equation}
    G :=
    \begin{pmatrix}
        1 & \langle \psi_Q,\,\psi_P\rangle & \langle \psi_Q,\,f\psi_P\rangle\\
        \langle \psi_Q,\,\psi_P\rangle & 1 & \langle \psi_P,\,f\psi_P\rangle\\
        \langle f \psi_P,\,\psi_Q\rangle & \langle f\psi_P,\,\psi_P\rangle & \langle f\psi_P,\,f\psi_P\rangle
    \end{pmatrix}.
\end{equation}
The crucial observation is that $G$ is positive semidefinite and thus has a non-negative determinant which can be viewed as a second degree polynomial $\pi(x)$ evaluated at $x=\langle\psi_Q,\,f\psi_P\rangle$ and is given by
\begin{equation}
    \mathrm{det}(G) =: \pi(x)\rvert_{x=\langle\psi_Q,\,f\psi_P\rangle}
\end{equation}
where $\pi(x) = ax^2 + bx + c$ is a polynomial with coefficients
\begin{align}
    \begin{gathered}
        a =-1,\hspace{1em}b=2\cdot\langle\psi_P,\,\psi_Q\rangle\cdot\langle\psi_P,\,f\psi_P\rangle\\
        c = (1-\abs{\langle\psi_P,\,\psi_Q\rangle}^2)\langle f\psi_P,\,f\psi_P\rangle - \langle f\psi_P,\,\psi_P\rangle^2.
    \end{gathered}
\end{align}
The non-negativity of $\mathrm{det}(G)$ implies that $\pi(x=\langle\psi_Q,\,f\psi_P\rangle) \geq 0$ and thus effectively restricts the values which $\langle \psi_Q,\,f\psi_P\rangle$ can take to be bounded within the square roots of $\pi$ so that
\begin{equation}
    \frac{b}{2} - \sqrt{\frac{b^2}{4} + c} \leq \langle \psi_Q,\,f\psi_P\rangle \leq \frac{b}{2} + \sqrt{\frac{b^2}{4} + c}.
\end{equation}
For positive functions $f \geq 0$, we can upper bound $\langle \psi_Q,\,f\psi_P\rangle$ via the Cauchy-Schwarz inequality and obtain a \emph{lower} bound on the expectation of $f$ under $Q$. Taking as our function $f$ to be the loss function of interest $f := \ell$, under the assumption that $\sup_{z\in Z}\abs{\ell(z)} \leq M$ for some $M > 0$, we can finally recast this lower bound as an upper bound on the expectation of $\ell$ under $Q$. Taking the supremum with respect to $Q$ leads to a bound on the worst-case risk~\eqref{eq:worst-case-risk}. We remark that in this way, we obtain both lower and upper bounds on the expected loss. As we will show, these bounds can be used to bound useful statistics, such as the accuracy or the AUC score used in binary classification.
In the following Theorem, we state our main result as an upper bound to the worst-case risk~\eqref{eq:worst-case-risk} and refer the reader to Appendix~\ref{apx:proof-lower-bound} for the analogous lower bound.
\begin{theorem}
    \label{thm:main}
    Let $\ell\colon\cZ \to \R_+$ be a loss function and suppose that $\sup_{z\in\cZ}\abs{\ell(z)} \leq M$ for some $M > 0$. Then, for any probability measure $P$ on $\cZ$ and $\rho > 0$ we have
    \begin{equation}
        \label{eq:main-generalization-bound}
        \begin{aligned}
            &\hspace{-.9em}\sup_{Q \in B_\rho(P)} \bE_{Q}[\ell(Z)] \leq \bE_{P}[\ell(Z)] + 2C_\rho\sqrt{\bV_{P}[\ell(Z)]} \\
            & + \rho^2(2-\rho^2)\bigg[M - \bE_{P}[\ell(Z)] - \frac{\bV_{P}[\ell(Z)]}{M - \bE_{P}[\ell(Z)]}\bigg]
        \end{aligned}
    \end{equation}
    where $C_\rho = \sqrt{\rho^2(1-\rho^2)^2(2-\rho^2)}$ and $B_\rho(P) = \{Q \in \cP(\cZ)\colon \, H(P,\,Q) \leq \rho\}$ is the Hellinger ball of radius $\rho$ centered at $P$. The radius $\rho$ is required to be small enough such that
    \begin{equation}
        \rho^2 \leq 1 - \left[1 + \frac{(M - \bE_{P}[\ell(Z)])^2}{\bV_{P}[\ell(Z)]}\right]^{-1/2}.
    \end{equation}
\end{theorem}
We refer the reader to Appendix~\ref{apx:proof-main-thm} for a full proof and now make some general observations about this result.
The bound~\eqref{eq:main-generalization-bound} presents a \emph{pointwise} guarantee in the sense that it upper bounds the distributional worst-case risk for a particular model $\ell(\cdot)$.
This is in contrast to bounds which hold uniformly for an entire model class and introduce complexity measures such as covering numbers and VC-dimension which are hard to compute for many
practical problems.
Other techniques which yield a pointwise robustness certificate of the form~(\ref{eq:main-generalization-bound}), typically express the uncertainty set as a Wasserstein ball around the distribution $P$~\cite{sinha2017certifying,shafieezadeh2019regularization,blanchet2019quantifying,cranko2021generalised}, and require the model $\ell$ to be sufficiently smooth. For example, the certificate presented in~\cite{sinha2017certifying} can only be tractably computed for small neural networks for which one can upper bound their smoothness by bounding the Lipschitz constant of their gradients. For more general and large-scale neural networks, these bounds quickly become intractable and/or lead to vacuous certificates.
For example, it is known that computing the Lipschitz constant of neural networks with ReLU activations is NP-hard~\cite{virmaux2018lipschitz}.
Secondly, we emphasize that our bound~\eqref{eq:main-generalization-bound} is ``faithful'', in the sense that, as the radius approaches zero, $\rho \to 0$, the bound converges towards the true expectation $\bE_P[\ell(Z)]$.
This is of course desirable for any such bound as it indicates that any intrinsic gap vanishes as the covered distributions become increasingly closer to the reference distribution $P$.
A third observation is that the bound~\eqref{eq:main-generalization-bound} is \emph{monotonically increasing} in the variance, indicating that low-variance models exhibit better generalization properties, which can be seen in light of the bias-variance trade-off.
More specifically, from the form our bound~\eqref{eq:main-generalization-bound} takes, we see that minimizing the variance-regularized objective $\mathcal{L}(\theta) = \bE_{Z\sim P}[\ell_\theta(Z)] + \lambda \bV_{Z\sim P}{\ell_\theta(Z)}$, effectively amounts to minimizing an upper bound on the worst-case risk.
Indeed, various recent works have highlighted the connection between variance regularization and generalization~\cite{lam2016robust, maurer2009empirical,gotoh2018robust,duchi2019variance} and our result provides further evidence for this observation.

\section{Certifying Out-of-domain Generalization}
Taking advantage of our weak assumptions on the loss functions and models, we now apply our framework to the problem of certifying
the out-of-domain generalization
performance of a given classifier, when measured in terms of different loss functions.
In practice, one is typically only given a finite sample $Z_1,\,\ldots,\,Z_n$ from the in-domain distribution $P$ and the bound~\eqref{eq:main-generalization-bound} needs to be estimated empirically.
To address this problem, our next step is to present a finite sampling version of the bound~\eqref{eq:main-generalization-bound} which holds with arbitrarily high probability over the distribution $P$.
Second, we instantiate our results with specific distribution shifts, namely, shifts in the label distribution, and shifts which only affect the covariates. Finally, we highlight specific loss and score functions and show how our result can be applied to certify the out-of-domain generalization of these functions.

\subsection{Finite Sample Results}
\label{subsec:finite-sampling}
Let $Z_1,\,\ldots,\,Z_n  \stackrel{iid}{\sim} P$ be an independent and identically distributed sample from the in-domain distribution $P$.
One immediate way to use our bound would be to construct the empirical distribution $\hat{P}_n$ and consider the worst-case risk over distributions $Q\in B_\rho(\hat{P}_n)$, while computing the bound on the right hand side of~\eqref{eq:main-generalization-bound} with the empirical mean and unbiased sample variance.
However, for $\rho < 1$, the Hellinger ball $B_\rho(\hat{P}_n)$ will in general only contain distributions with discrete support since any continuous distribution $Q$ has distance $1$ from $\hat{P}_n$.
We therefore seek another path and make use of concentration inequalities for the population variance and mean, in order to get statistically sound guarantees which hold with arbitrarily high probability. To achieve this, we bound the expectation value via Hoeffding's inequality~\cite{hoeffding1963probability}, and the population variance via a bound presented in~\cite{maurer2009empirical}. In the second step, we use the union bound as a means to bound both variance and expectation simultaneously with high probability.
We leave the derivation and proof to Appendix~\ref{subsec:finite-sampling}.
These ingredients lead to the finite sampling-based version of Theorem~\ref{thm:main}, which we state in the following Corollary.
\begin{corollary}[Finite-sampling bound]
    \label{cor:finite-sampling-bound}
    Let $Z_1,\,\ldots,\,Z_n$ be independent random variables drawn from $P$ and taking values in $\cZ$. For a loss function $\ell\colon\cZ \to [0,\,M]$, let $\hat{L}_n:=\frac{1}{n}\sum_{i=1}^n\ell(Z_i)$ be the empirical mean and $S_n^2:=\frac{1}{n(n-1)}\sum_{1 \leq i < j \leq n}^n (\ell(Z_i) - \ell(Z_j))^2$ be the unbiased estimator of the variance of the random variable $\ell(Z)$, $Z\sim P$.
    Then, for any $\delta \in (0, 1)$, with probability at least $1 - \delta$
    \begin{equation}
        \label{eq:main-generalization-bound-sampling}
        \begin{aligned}
            &\sup_{Q \in B_\rho(P)} \bE_{Q}[\ell(Z)] \leq \hat{L}_n + 2C_\rho\sqrt{S_n^2} + \Delta_{n,\rho}\\
            &\hspace{2em} + \rho^2(2-\rho^2)\Bigg[M - \hat{L}_n\\
            &\hspace{4em} + \frac{S_n^2 + 2M\sqrt{\frac{2S_n^2\ln 2 /\delta}{n-1}} + \frac{2M^2\ln 2 / \delta}{n-1}}{\hat{L}_n - M \left(1 - \sqrt{\frac{\ln 2 / \delta}{2n}} \right)}\Bigg]
        \end{aligned}
    \end{equation}
    where
    \begin{equation}
        \Delta_{n,\rho} = \left(\frac{2C_\rho}{\sqrt{n-1}} - \frac{\rho^2(2-\rho^2)}{2\sqrt{n}}\right)M\sqrt{2\ln 2/\delta}
    \end{equation}
    and $C_\rho = \sqrt{\rho^2(1-\rho^2)^2(2-\rho^2)}$. The radius $\rho$ is required to be small enough such that
    \begin{equation}
        \rho^2 \leq 1 - \left[1 + \left(\frac{\hat{L}_n - M\left(1 - \sqrt{\frac{\ln 2/\delta}{2n}}\right)}{\sqrt{S_n^2} + M\sqrt{\frac{2\ln 1/\delta}{n - 1}}}\right)^2\right]^{-1/2}.
    \end{equation}
\end{corollary}
Thus, we have derived a certificate for out-of-domain generalization for general bounded loss functions and models $h$ which can be efficiently estimated from finite data sampled from the distribution $P$.

\subsection{Specific Distribution Shifts}
\label{sec:distribution-shifts}
We now consider specific distribution shifts and discuss our main results in light of shifts in the distributions of labels and covariates.

\subsubsection{Label Distribution Shifts}
\label{subsec:label-shifts}
Shifts in the label distribution occur when, during deployment, an ML-system operates in an environment where the relative frequency of certain classes increases or decreases, compared to the training environment, or, as is common in practical applications, instances of previously unseen classes appear. This can potentially harm the model performance dramatically and can have severe implications, in particular in the context of fairness and ethics in machine learning. To investigate this type of distribution shift, we follow the
common practice to assume
that the distribution over covariates, conditioned on the labels, stays constant.
Formally, here, we consider the distribution shift $P \to Q$ expressed via
\begin{equation}
    p(x,\,y) = \pi(x\lvert\,y)p(y) \mapsto q(x,\,y) = \pi(x\lvert\,y)q(y)
\end{equation}
where $\pi(x\lvert\,y)$ is given by a fixed distribution over covariates, conditioned on labels.
In this case, it can be shown that the Hellinger distance is equal to the $L_2$ norm between the square roots of the (label) probability vectors $p = (p(1),\,\ldots,\,p(K))^T\in\R^K$ and $q = (q(1),\,\ldots,\,q(K))^T\in\R^K$ where $K$ is the number of classes, so that
\begin{equation}
    H(P,\,Q) = \frac{1}{\sqrt{2}}\norm[2]{\sqrt{{p}} - \sqrt{{q}}}
\end{equation}
where the square root is applied to each element in the respective probability vector.

\subsubsection{Covariate Distribution Shifts}
\label{subsec:covariate-shifts}
In contrast to label distribution shifts, here we consider shifts to the distribution of covariates. This models scenarios where the relative frequency of labels stays constant, but environments change, for example the shift from day to night in autonomous driving or wildlife surveillance.
Formally, we consider the shift $P \to Q$ with
\begin{equation}
    p(x,\,y) = \pi(y\lvert\,x)p(x) \mapsto q(x,\,y) = \pi(y\lvert\,x)q(x)
\end{equation}
where $\pi(y\lvert\,x)$ is given by a fixed distribution on labels, conditioned on the covariates.
In this scenario, the Hellinger distance between $P$ and $Q$ reduces to the distance between the marginals
\begin{equation}
    H(P,\,Q) = \sqrt{\frac{1}{2}\int_\cX \Big(\sqrt{p(x)} - \sqrt{q(x)}\Big)^2\,dx}.
\end{equation}
In principle, this quantity could be estimated from unlabeled samples of a target distribution $Q$, enabling one to reason about distributional robustness of a given model, by evaluating our bounds from Theorem~\ref{thm:main} and Corollary~\ref{cor:finite-sampling-bound}.
However, in practice, it is generally difficult to estimate $f$-divergences, and in particular the Hellinger distance, from data for practically relevant problem instances.
Although first steps in this direction have been made~\cite{nguyen2007nonparametric,nguyen2010estimating,sreekumar2021non}, it remains largely an open problem and a potential solution would give our approach additional ounces of practical significance.
We view this problem as orthogonal to certifying out-of-domain generalization and believe that research efforts towards such an end-to-end solution pose an exciting future research direction.

\paragraph{Discussion.}
We notice that when considering label- or covariate distribution shifts, we are effectively interested in a subset of all probability distributions with a given predefined Hellinger distance. In other words, if the shift $P \to Q$ models the label distribution shift with distance $H(P,\,Q) \leq \rho$, then applying the certificate \eqref{eq:main-generalization-bound} with radius $\rho$ also covers every other type distribution shift bounded by $\rho$ and hence gives a more conservative view than desired. This is because, in general,
\begin{equation}
    \sup_{\substack{Q: H(P,\,Q) \leq \rho \\ q(\cdot\lvert\,y) \equiv p(\cdot\lvert\,y)}} \bE_Q[\ell(Z)]
    \leq \sup_{Q: H(P,\,Q) \leq \rho} \bE_Q[\ell(Z)]
\end{equation}
arising from the additional constraint that $q(x\lvert\,y) = p(x\lvert\,y)$ for all $x\in\cX$.
A similar argument can be made for covariate shifts.
Naturally, this leads to an intrinsic gap between the actual and certified robustness, which we also observe in our experiments.
Finally, it is worth pointing out the connection with generalization from finite amounts of data which can be seen as a specific instantiation of the worst-case risk~\eqref{eq:worst-case-risk} where the in-domain distribution corresponds to the empirical distribution $\hat{P}_n$ and the radius $\rho$ decays as $\cO(1/n)$. In this sense, the distribution shift originates from the transition from the empirical to the true data distribution. This type of distribution shift has been analyzed in~\cite{duchi2019variance} where further links to variance-based regularization have been established.

\subsection{Specific Loss and Score Functions}
We now turn our attention to specific loss functions and discuss, in particular, the Jensen-Shannon divergence loss, the classification error, and the AUC score.

\subsubsection{Jensen-Shannon Divergence}
\label{apx:jensen-shannon-divergence}
The Jensen-Shannon Divergence is a particular type of loss function for classification models, and serves as a symmetric alternative to other common losses such as cross entropy.
It has been observed that the JSD loss and its generalizations have favorable properties compared to the standard cross entropy loss, as it is bounded, symmetric, and its square root is a distance and hence satisfies the triangle inequality. In~\cite{englesson2021generalized} it has been observed that JSD loss can be seen as an interpolation between cross-entropy and mean absolute error and is particularly well suited for classification problems with noisy labels. Formally, the Jensen-Shannon divergence is defined as
\begin{equation}
    D_{JS}(P,\,Q) := \frac{1}{2}\Big(D_{KL}(P\|\,\mu) + D_{KL}(Q\|\,\mu)\Big)
\end{equation}
where $D_{KL}$ is the Kullback-Leibler divergence and $\mu = \frac{1}{2}(P + Q)$. Since it is a bounded loss function, it is straightforward to apply our results to certify the out-of-domain generalization for the JSD loss and, due to its smoothness, allows for a principled comparison between our bound and the Wasserstein distance certificates proposed in~\cite{sinha2017certifying,cranko2021generalised}.

\subsubsection{Classification Error}
\label{sec:classification-error}
The classification error is among the most popular choices for measuring the performance of classification models and serves as a means to assess how accurate a classifier is on a given data distribution. As it is a non-smooth function, existing approaches cannot in general certify distributional robustness for this function. In contrast, one can immediately instantiate our Theorem~\ref{thm:main} (or the finite sampling version from Corollary~\ref{cor:finite-sampling-bound}) with this loss.
Indeed, for a fixed model $h\colon\cX\to\cY$, let $\epsilon_P := \bP_{(X,Y)\sim P}[h(X) \neq Y]$ and analogously $\epsilon_Q$. Then, in the infinite sampling regime, we immediately get an upper bound on the worst-case classification error from Theorem~\ref{thm:main}. Namely, for a sufficiently small radius $\rho^2 \leq 1 - \sqrt{\epsilon_P}$, we have
\begin{equation}
    \begin{aligned}
        \sup_{Q\in B_\rho(P)}\epsilon_Q &\leq \epsilon_P + 2C_\rho\sqrt{\epsilon_P(1-\epsilon_P)}\\
        &\hspace{6em}+ \rho^2(2-\rho^2)(1-2\epsilon_P)
    \end{aligned}
\end{equation}
where $C_\rho = \sqrt{\rho^2(1-\rho^2)^2(2-\rho^2)}$.

\subsubsection{AUC Score}
\label{subsec:auc-score}
Among other uses, the Area under the ROC (AUC) score~\cite{hanley1982meaning,clemenccon2008ranking} is a metric to measure the performance of binary classification models. Unlike the classification error, which captures the ability to classify a single randomly chosen instance, the AUC score provides a means to quantify the ability to correctly assigning to any positive instance a higher score than to a randomly chosen negative instance. For a binary classification model $h\colon\cX \to \R$ that outputs the score of the positive class, the AUC score is defined as
\begin{equation}
    \AUC(h) = \bP\left[h(X) \geq h(X')\lvert Y = 1,\,Y'=-1\right]
\end{equation}
where $(X,\,Y)$ and $(X',\,Y')$ are independent and identically distributed according to $P$.
By introducing the notation $X_\pm := X\lvert\,Y=\pm1$, we can equivalently write the AUC score as an expectation value over the joint (conditional) distribution of $Z := (X_+,\,X_-)$
\begin{equation}
    \AUC(h) = \bE_{(X_+,\,X_-)\sim P_Z}[\Id_{\{h(X_+) \geq h(X_-)\}}].
\end{equation}
We notice that only distribution shifts on the covariates have an impact on the AUC score. For this reason, we consider a setting similar to the covariate shift setting of~\secref{subsec:covariate-shifts}, although we consider shifts in the conditional distribution $p(x\lvert\,y) \mapsto q(x\lvert y)$ for each $y \in\{\pm1\}$ in contrast to shifts in the marginals. Due to independence, the probability density function of $Z\sim P_Z$ can be written as
\begin{equation}
    p_Z(x_+,\,x_-) = p(x\lvert\,y=+1) p(x\lvert\,y=-1)
\end{equation}
and similarly for the shifted distribution $Q$. Thus, assuming that for both negative and positive samples a distribution drift with $H(P_{X\lvert Y=y},\,Q_{X\lvert Y=y}) \leq \rho$ occurs, the squared Hellinger distance between $P_Z$ and $Q_Z$ is bounded by
\begin{equation}
    \begin{aligned}
        H^2(P_Z,\,Q_Z) \leq \rho^2(2-\rho^2).
    \end{aligned}
\end{equation}
Thus, for certifying out-of-domain generalization for the AUC score, we can apply our bound by instantiating it with Hellinger distance $\sqrt{\rho^2(2-\rho^2)}$. We remark that for the AUC score, one is typically interested in \emph{lower} bounding it under distribution shifts. To that end, we present a lower bound version of our Theorem~\ref{thm:main} in Appendix~\ref{apx:proof-lower-bound}.

\section{Experiments}
\begin{figure*}[t]
    \vskip 0.2in
    \centering
    \includegraphics[width=\textwidth]{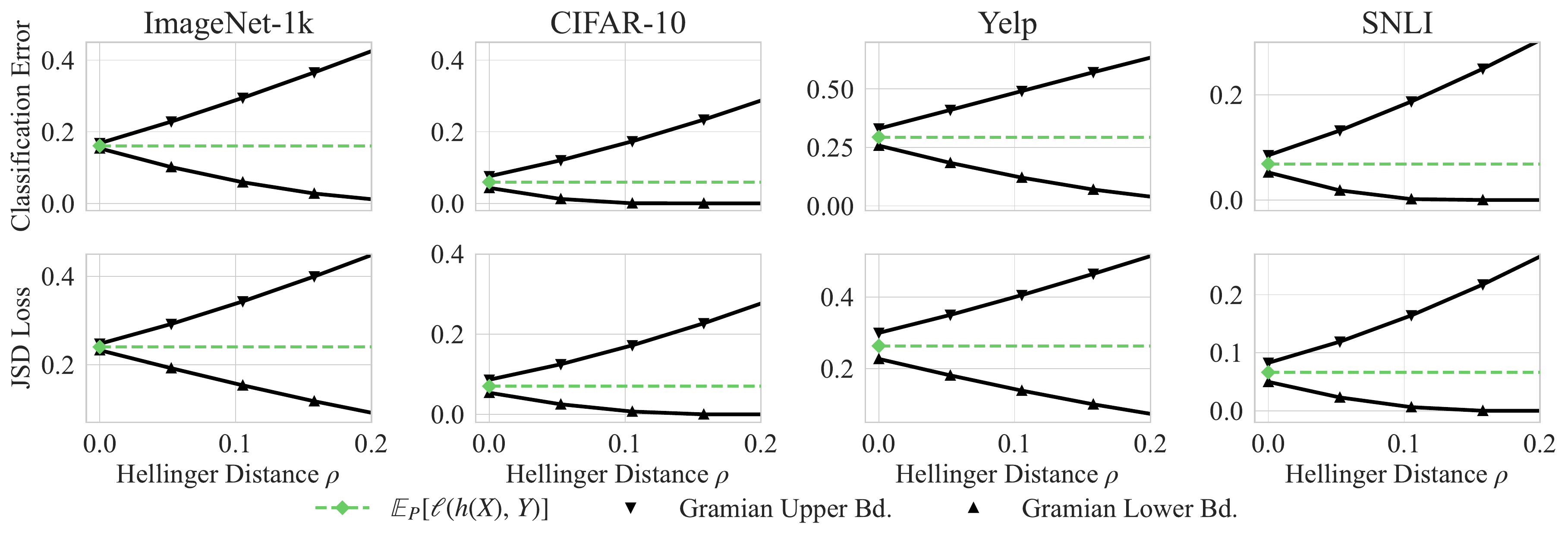}
    \vspace{-2em}
    \caption{Distributional robustness certificates for generic distribution shifts on vision and NLP datasets for JSD and 0-1 loss.}
    \label{fig:generic-shift-all}
    \vskip -0.2in
\end{figure*}
\label{sec:experiments}

We now experimentally validate our theoretical findings on
a diverse collection of
datasets and scenarios.
We first provide
certificates considering
generic distribution shifts $P\to Q$ and then
provide detailed analysis on the two specific scenarios
described in Sections~\ref{subsec:label-shifts} and~\ref{subsec:covariate-shifts}, namely, shifts in the label and in the covariate distributions.
Finally, we construct a synthetic example that allows for a fair comparison of our bounds with the Wasserstein certificate of~\cite{sinha2017certifying}, which indicates that in addition to favorable scalability properties, our bounds are also considerably tighter.
We remark that all our bounds are computed using the finite sampling bounds presented in Corollary~\ref{cor:finite-sampling-bound} and hold with $99\%$ probability ($\delta = 0.01$).\footnote{Our code is publicly available at \url{https://github.com/DS3Lab/certified-generalization}.}

\paragraph{Datasets} We certify out-of-domain generalization on two standard vision datasets:
ImageNet-1k~\cite{russakovsky2015imagenet} containing objects of 1,000 different classes;
and CIFAR-10~\cite{cifar10}, which contains natural images of 10 different classes.
We also conduct experiments on the
standard natural language processing (NLP) datasets Yelp~\cite{yelpdataset} and SNLI~\cite{snliemnlp2015}.
We follow \citet{lin2017structured}
to sample $2,000$ examples for the Yelp test set and
$10,000$ examples for the
SNLI test set.

\paragraph{Models}
For classification on ImageNet-1k, we use the EfficientNet-B7~\cite{tan2019efficientnet} architecture which
we initialize with pre-trained weights;
we use DenseNet-121~\cite{huang2017densely}
for CIFAR-10.
On Yelp, we use BERT~\cite{devlin2018bert} and on SNLI we use a DeBERTa architecture~\cite{he2020deberta}.

\begin{figure}[t]
    \vskip 0.2in
    \centering
    \includegraphics[width=\columnwidth]{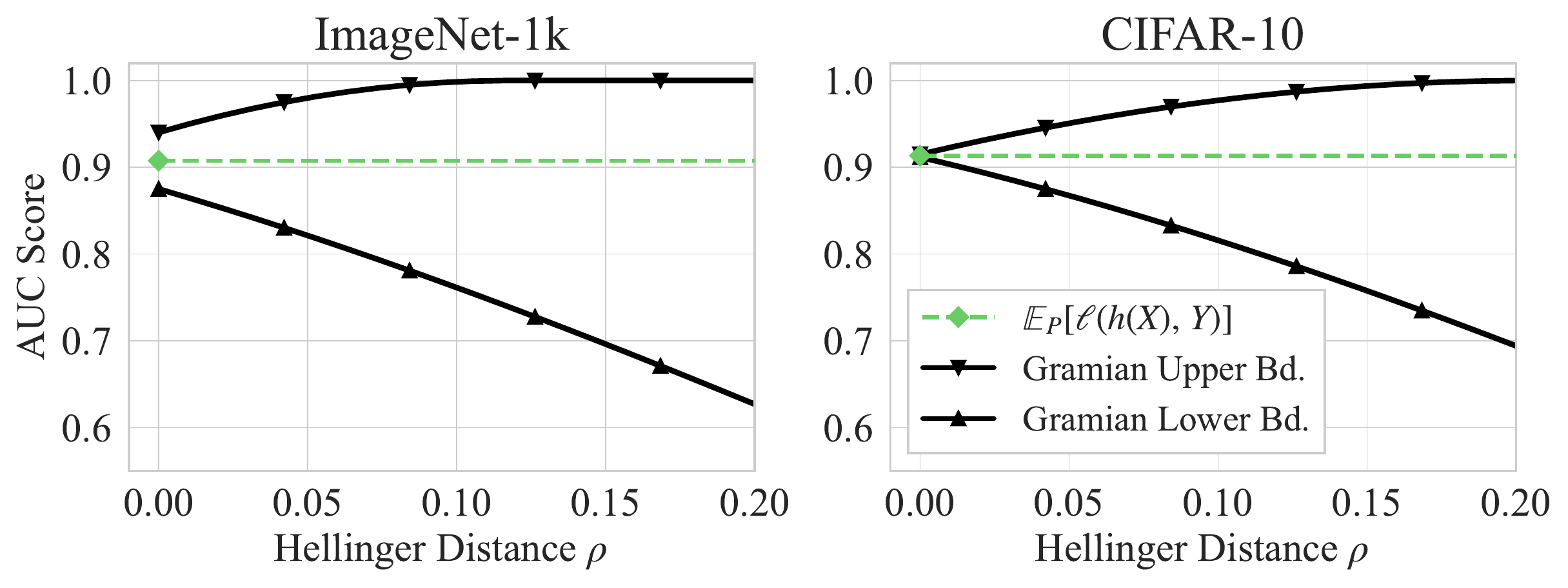}
    \vspace{-1em}
    \caption{Distributional robustness certificates for AUC against generic distribution shifts on binary ImageNet and CIFAR datasets.}
    \label{fig:generic-shift-auc}
    \vskip -0.2in
\end{figure}

\paragraph{Settings for AUC Scores}
When we consider AUC scores, we further constrain all multiclass datasets into a binary version.
To this end, on ImageNet, we randomly choose two classes and train a ResNet-152 architecture to discriminate between the two Synsets \emph{n01601694} and \emph{n04330267} (corresponding to the classes `water ouzel' and `stove'). Similarly, on CIFAR-10 we also pick two classes at random and train a ResNet-110 classifier for the two classes \emph{`bird'} and \emph{`horse'}.

\subsection{Certifying Distribution Shifts}

Figures~\ref{fig:generic-shift-all}
and~\ref{fig:generic-shift-auc}
illustrate the certificates that
we provide on a diverse range of
datasets, considering three
different scores: classification
error, JSD loss, and AUC score.
In all these figures, the
x-axis corresponds to the degree
of distribution drift,
and the \textit{Gramian Certificate}
curves correspond to
the lower and upper bound
of these scores under
distribution drifts.
To our best knowledge,
this is the first time that nonvacuous
certificates are obtained
on this diverse
range of datasets, scores, and
large-scale models.

\begin{figure}[t]
    \vskip 0.2in
    \centering
    \includegraphics[width=0.95\columnwidth]{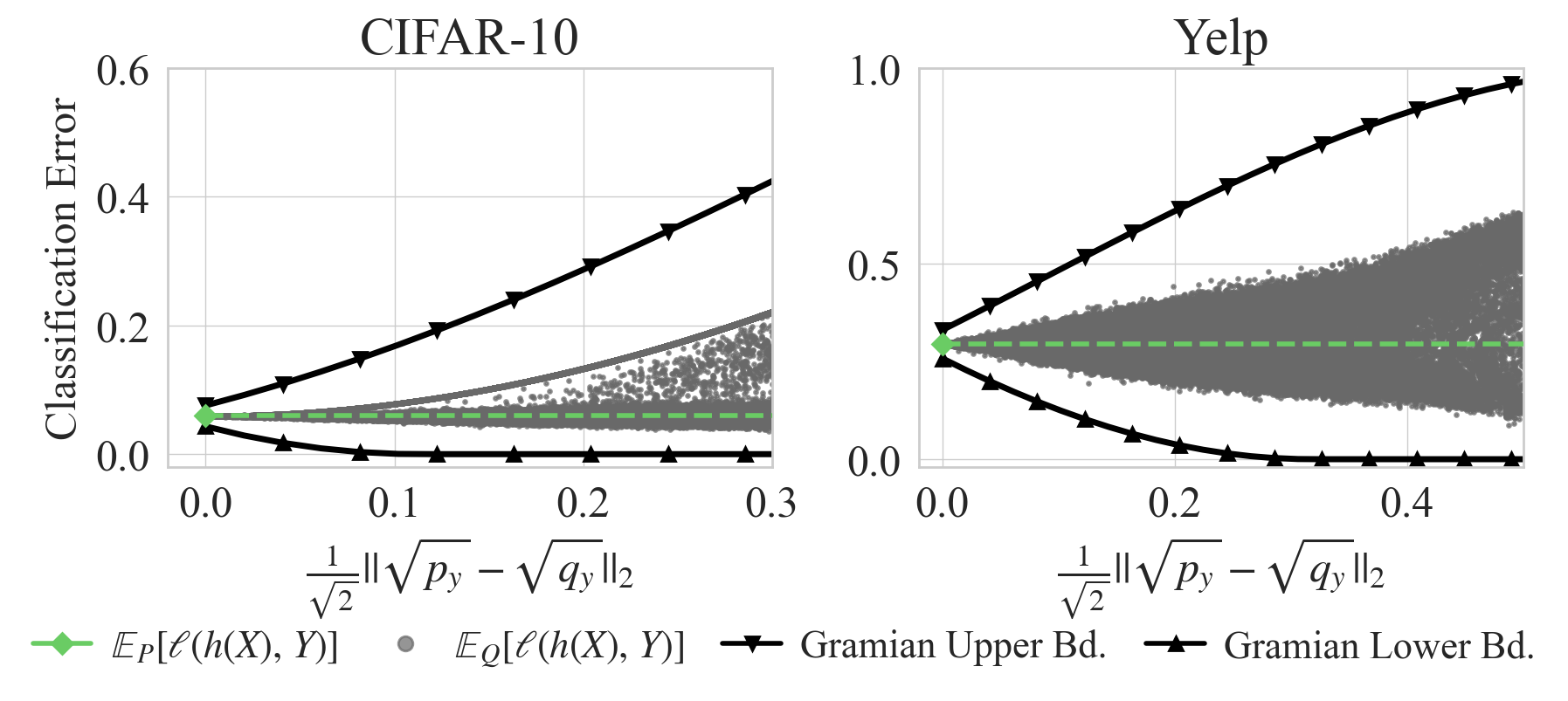}
    \vspace{-1em}
    \caption{Certified Generalization for label distribution shifts.
    Each gray point corresponds to a randomly sampled label distribution with corresponding Hellinger distance and empirical loss.}
    \label{fig:label-shift}
    \vskip -0.2in
\end{figure}

\paragraph*{Label Distribution Shifts}

To get a better indication of how well our certificates capture the true risk under label distribution shifts,
we randomly generate 100,000 shifted class distributions on the CIFAR-10 and Yelp datasets by 1) subsampling existing classes, 2) removing the counts of existing classes, and 3) including new "unseen" classes.
This allows us to empirically compute both the classification error and the Hellinger distance
and enables us to compare the certificates to the actual loss on the shifted distribution.
We can see from Figure~\ref{fig:label-shift} that our certificates
indeed provide a valid upper
and lower bound.
Note that, given that
all shifted class distributions
are randomly sampled,
we might not hit the
true worst-case scenario,
explaining the clear gap between the generalization certificates and the scores obtained from the randomly generated label distributions.
Another reason for the gap can be
attributed to the
intrinsic gap for label and covariate shifts, discussed in Section~\ref{sec:distribution-shifts}.
We refer the reader to Appendix~\ref{apx:additional-models} for analogous figures with a larger set of model architectures on the CIFAR-10 dataset. Finally, we point out the difficulty in sampling these class distributions for datasets with a large number of classes and include analogous figures for ImageNet and the SNLI dataset in Appendix~\ref{apx:additional-functions}.

\begin{figure}[t]
\vskip 0.2in
\centering
     \begin{subfigure}[b]{0.24\textwidth}
         \centering
         \includegraphics[width=\columnwidth]{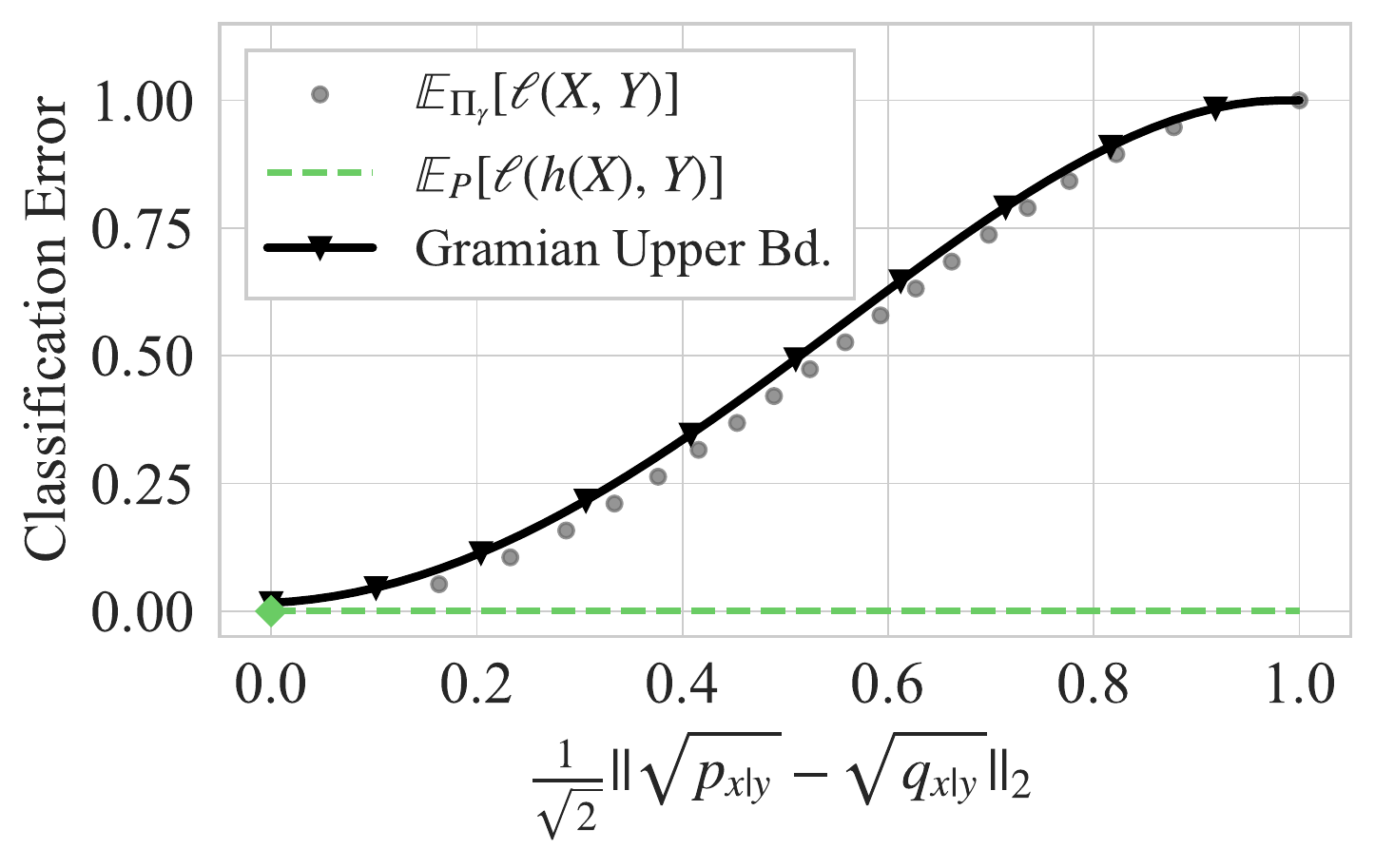}
         \vspace{-1em}
         \caption{Classification Error}
     \end{subfigure}%
     \hfill
     \begin{subfigure}[b]{0.24\textwidth}
         \centering
         \includegraphics[width=\columnwidth]{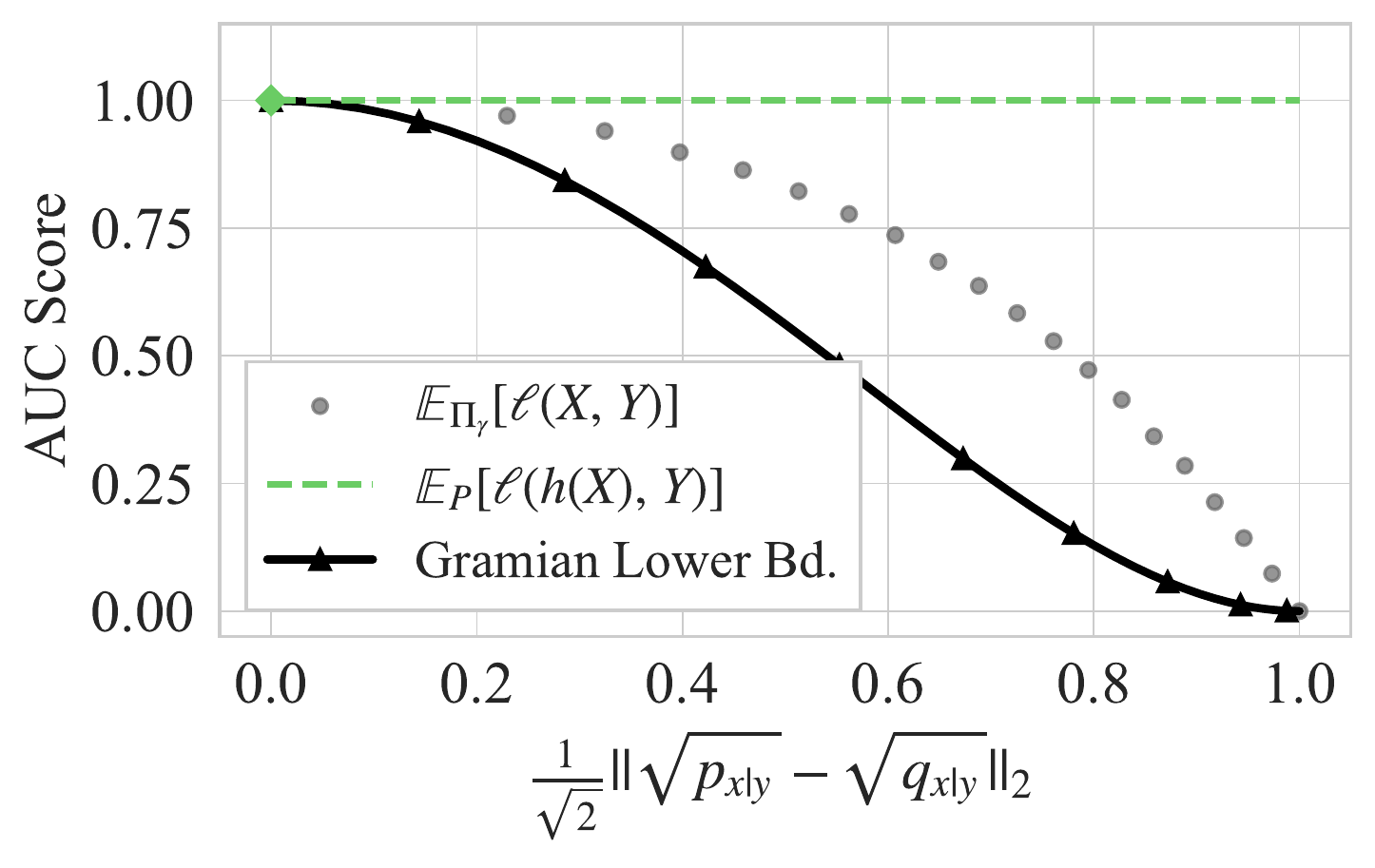}
         \vspace{-1em}
         \caption{AUC Score}
     \end{subfigure}
\vspace{-2em}
\caption{Certificate against covariate shift on colored MNIST.}
\label{fig:cmnist-covariate-drift}
\vskip -0.2in
\end{figure}

\paragraph*{Covariate Distribution Shifts.}

We now investigate our certificates in light of changes in the distribution of the covariates and consider the scenario described in~\secref{subsec:covariate-shifts}.
In this experiment, we use the binary Colored MNIST dataset~\cite{kim2019learning,arjovsky2019invariant}, which is constructed from the MNIST dataset by coloring the digits 0-4 in green and 5-9 in red for the training set, while flipping the coloring in the test set. The classifier is then trained to classify the digits into the two groups $\{0,1,2,3,4\}$ and $\{5,6,7,8,9\}$. In this setting, the classifier learns to perfectly distinguish the two classes in the training set, but fails on the testing set since the color is a stronger predictor than the shape of the digits.
To investigate the space between these two extreme cases, we generate mixture distributions between training and test distribution in the following way.
We set $P$ to be the training distribution and $Q$ the testing distribution (containing digits with flipped colors).
Guided by a mixing parameter $\gamma$, we mix $P$ and $Q$ to obtain the mixture distribution $\Pi_\gamma := \gamma\cdot P + (1-\gamma)\cdot Q$.
Since $P$ and $Q$ have disjoint support, we compute the Hellinger distance between $P$ and $\Pi_\gamma$ as $H(P,\,\Pi_\gamma) = \sqrt{1 - \sqrt{\gamma}}$ as shown in Appendix~\ref{apx:hellinger-distance-mixture}.
Figure~\ref{fig:cmnist-covariate-drift} illustrates our robustness certificates for the 0-1 loss and the AUC Score, as well as the empirical losses $\bE_{\Pi_\gamma}[\ell(X, \,Y)]$ for different values of the mixture parameter $\gamma$.
We see from the figure that our technique provides quite tight certificates for both classification error and AUC score.

\begin{figure}[t!]
\vskip 0.2in
\centering
    \begin{subfigure}[b]{0.48\textwidth}
         \centering
         \includegraphics[width=\columnwidth]{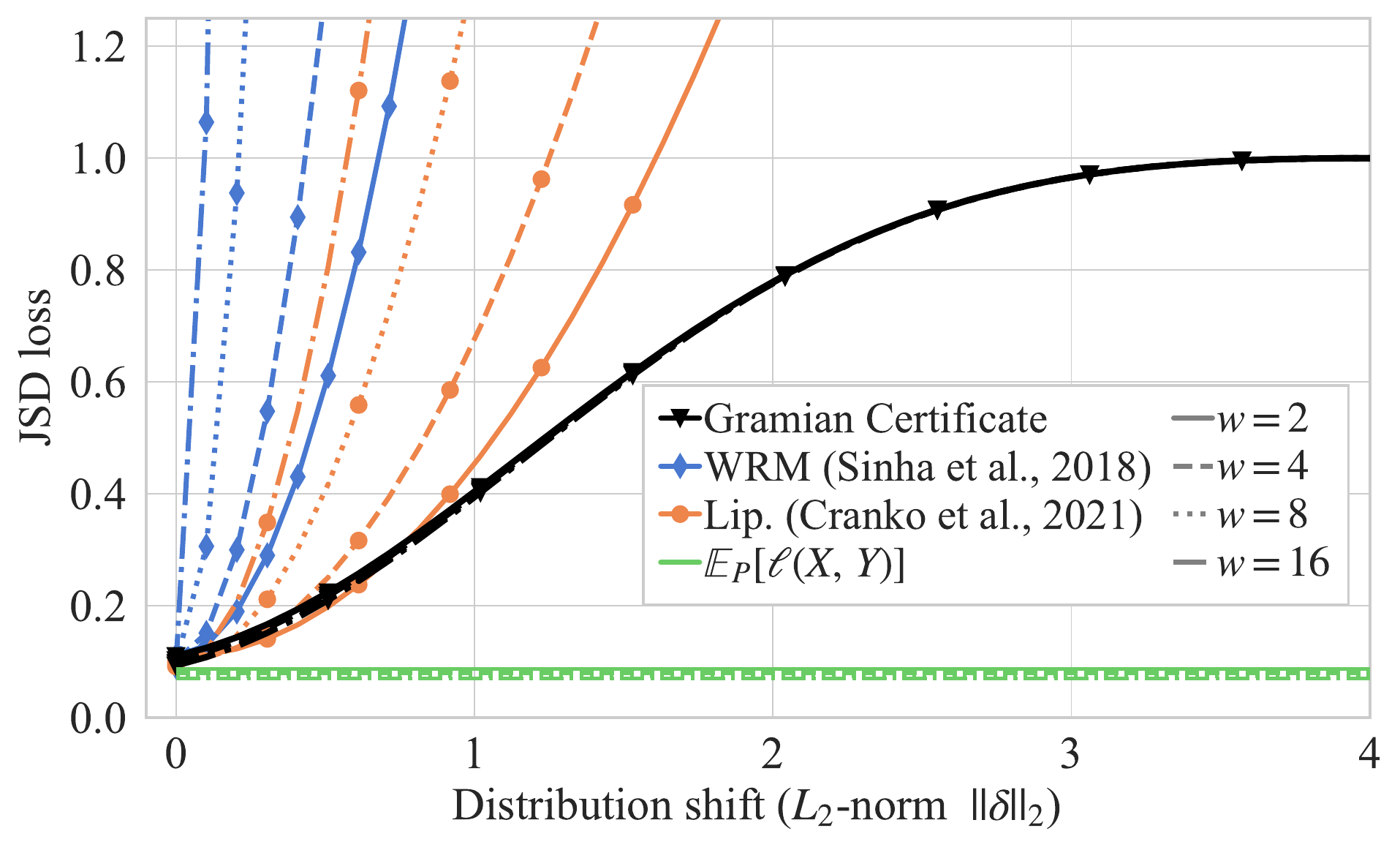}
         \vspace{-0.5em}
         \caption{Robustness certificates for varying network widths $w$ and $n_h=2$ hidden layers.}
     \end{subfigure}
     \begin{subfigure}[b]{0.48\textwidth}
         \centering
         \includegraphics[width=\columnwidth]{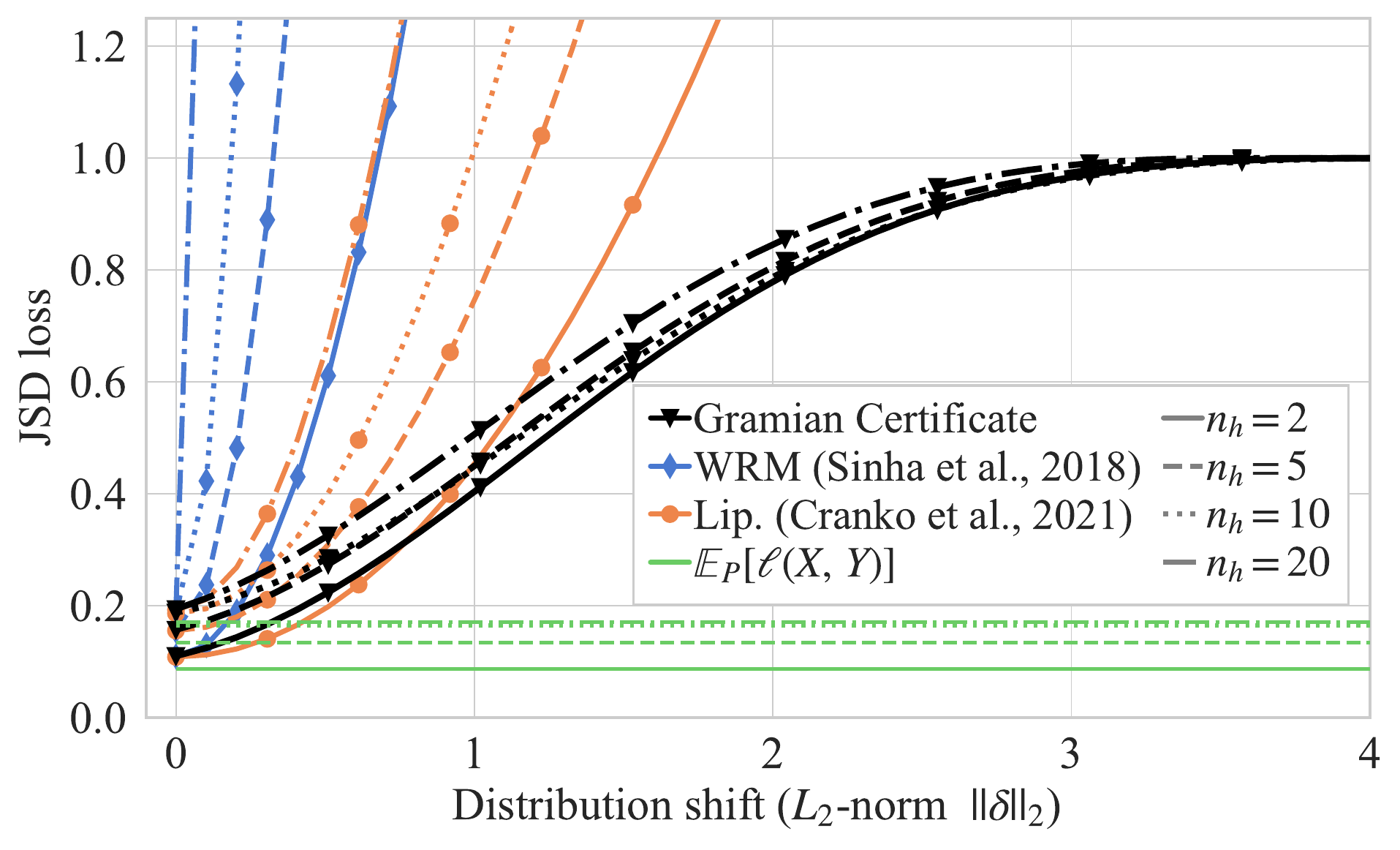}
         \vspace{-1em}
         \caption{Robustness certificates for varying numbers of hidden layers $n_h$ and fixed width $w=4$.}
         \vspace{-0.5em}
\end{subfigure}
\caption{Comparison of our approach with the Wasserstein-based certificates from~\cite{sinha2017certifying,cranko2021generalised} for varying levels of model complexity.}
\label{fig:wrm-comparison}
\vskip -0.2in
\end{figure}

\subsection{Comparison with Wasserstein Certificates}
We now construct a synthetic example that enables a fair comparison with two baseline certificates based on the Wasserstein distance. Namely, we compare our approach with 1) the certificate which uses the Lipschitz constant of the ML model, presented in~\cite{cranko2021generalised}; and 2) with the pointwise robustness certificate derived in~\cite{sinha2017certifying} from the dual formulation of the worst-case risk. We remark that these certificates cannot be applied to our previous examples because of their prohibitive assumptions. To make the three techniques comparable, we consider a Gaussian mixture model and certify the Jensen-Shannon divergence loss, while modeling distribution shifts as dislocations, $X \mapsto X + \delta$ for a fixed perturbation vector $\delta$. This allows us to parameterize the distribution shift via the $L_2$-norm of $\delta$ and obtain a one-to-one correspondence between our Hellinger distance and the Wasserstein distance, and enables a principled comparison. We describe the details of this synthetic dataset in Appendix~\ref{apx:synthetic-dataset}.
To investigate how the techniques scale with increased model complexity, we use fully connected feedforward neural networks with varying depths and widths. In addition, to accommodate \cite{sinha2017certifying}'s assumptions on smoothness, we use ELU activation functions on all layers.
We remark that the bound in~\cite{sinha2017certifying} requires one to solve a complex maximization problem, which requires the composition of the loss function and the network to be sufficiently smooth. Furthermore, the concavity of the maximization problem hinges on knowledge of the Lipschitz constant of the gradient. For small examples, this Lipschitz constant can be obtained, as we show in Appendix~\ref{apx:wrm-comparison-lipschitz} for the JSD loss function.
As can be seen in Figure~\ref{fig:wrm-comparison}, all bounds converge to the expected loss $\bE_P[\ell(X,\,Y)]$ as the perturbation goes to zero, $\norm[2]{\delta} \to 0$. However, the certificate from~\cite{sinha2017certifying} quickly becomes vacuous as the perturbation magnitude increases. In addition, both baseline bounds become loose with increasing model complexity, while our bound is virtually agnostic to the model architecture as it only depends on the variance and expected loss on the distribution $P$.

\section{Related Work}
Distributionally robust optimization first appeared in the context of inventory management~\cite{scarf1958min} and has since been discovered by the machine learning community as a useful tool to train machine learning models which generalize better to new distributions~\cite{ben2013robust,gao2016distributionally,shafieezadeh2019regularization}.
The uncertainty set occurring in the distributionally robust loss has been studied in terms of Wasserstein balls in~\cite{gao2016distributionally,sinha2017certifying,shafieezadeh2019regularization,cranko2021generalised,lee2017minimax,cisse2017parseval,kuhn2019wasserstein,blanchet2019quantifying}, and $f$-divergence balls in~\cite{ben2013robust,duchi2021statistics,lam2016robust,duchi2019variance,duchi2021learning}.
From a more general viewpoint,~\cite{husain2020distributional} connects integral probability metrics with distributional robustness in general and provides links with generative adversarial networks.
In another vein, maximum mean discrepancy measures have been investigated in~\cite{staib2019distributionally} for generalization in Kernel methods.
~\cite{sinha2017certifying} propose a method to certify generalization by using the dual formulation of the Wasserstein worst-case risk. However, their approach requires the loss model and loss function to be smooth and relies on an estimate of the Lipschitz constant of gradients, which quickly becomes vacuous for large problem sizes. Related techniques based on Wasserstein distances~\cite{gao2016distributionally,shafieezadeh2019regularization,blanchet2019quantifying,kuhn2019wasserstein,cranko2021generalised} make similarly prohibitive assumptions and generally fail to provide scalable alternatives.
In contrast, we study uncertainty sets expressed as Hellinger balls and provide a model-specific distributional robustness guarantee which only makes minimal assumptions on the loss (namely, boundedness) and thus scales to large problems.
The authors in~\cite{subbaswamy2021evaluating} consider distributionally robust optimization under fine-grained shifts in the marginal distributions, and reason about the worst-case risk on subpopulations in the data distribution.
Orthogonal to our work is the topic of certified adversarial robustness~\cite{wong2018scaling,lecuyer2019certified,cohen2019certified, szegedy2014intriguing,carlini2017adversarial}. This line of research seeks to reason about robustness at the instance level, while we aim to bound the worst-case risk over a set of distributions.

\section{Conclusion}
\label{sec:discussion}
In this paper, we have studied the problem of certifying the out-of-domain generalization for blackbox functions.
To that end, we have presented a framework to bound the worst-case population risk over an uncertainty set of probability distributions given by a Hellinger ball.
In contrast to existing approaches, our framework is scalable since it treats the loss function together with the model as a blackbox and thus requires virtually no knowledge about the internals of, e.g., neural networks.
We have provided experimental evidence that our technique can handle large models and datasets and provides, to the best of our knowledge, the first non-vacuous out-of-domain generalization bounds for problems as large as ImageNet with a full-fledged EfficientNet-B7.
While our techniques provide a means to certify robustness against \emph{general} distribution shifts, future research directions can potentially extensively study more specific distribution shifts.
In addition, it will be interesting to link our results to related topics such as fairness in machine learning.

\section*{Acknowledgments}
CZ and the DS3Lab gratefully acknowledge the support from the Swiss State Secretariat for Education, Research and Innovation (SERI)’s Backup Funding Scheme for European Research Council (ERC) Starting Grant TRIDENT (101042665), the Swiss National Science Foundation (Project Number 200021\_184628, and 197485), Innosuisse/SNF BRIDGE Discovery (Project Number 40B2-0\_187132), European Union Horizon 2020 Research and Innovation Programme (DAPHNE, 957407), Botnar Research Centre for Child Health, Swiss Data Science Center, Alibaba, Cisco, eBay, Google Focused Research Awards, Kuaishou Inc., Oracle Labs, Zurich Insurance, and the Department of Computer Science at ETH Zurich.
BL, LL and BW are supported by NSF grant No.1910100, NSF CNS 2046726, C3 AI, and the Alfred P. Sloan Foundation

\bibliography{cr-refs}

\begin{thebibliography}{51}
\providecommand{\natexlab}[1]{#1}
\providecommand{\url}[1]{\texttt{#1}}
\expandafter\ifx\csname urlstyle\endcsname\relax
  \providecommand{\doi}[1]{doi: #1}\else
  \providecommand{\doi}{doi: \begingroup \urlstyle{rm}\Url}\fi

\bibitem[AlBadawy et~al.(2018)AlBadawy, Saha, and Mazurowski]{albadawy2018deep}
AlBadawy, E.~A., Saha, A., and Mazurowski, M.~A.
\newblock Deep learning for segmentation of brain tumors: Impact of
  cross-institutional training and testing.
\newblock \emph{Medical physics}, 45\penalty0 (3):\penalty0 1150--1158, 2018.

\bibitem[Arjovsky et~al.(2019)Arjovsky, Bottou, Gulrajani, and
  Lopez-Paz]{arjovsky2019invariant}
Arjovsky, M., Bottou, L., Gulrajani, I., and Lopez-Paz, D.
\newblock Invariant risk minimization.
\newblock \emph{arXiv preprint arXiv:1907.02893}, 2019.

\bibitem[Beery et~al.(2018)Beery, Van~Horn, and Perona]{beery2018recognition}
Beery, S., Van~Horn, G., and Perona, P.
\newblock Recognition in terra incognita.
\newblock In \emph{Proceedings of the European conference on computer vision
  (ECCV)}, pp.\  456--473, 2018.

\bibitem[Ben-Tal et~al.(2013)Ben-Tal, Den~Hertog, De~Waegenaere, Melenberg, and
  Rennen]{ben2013robust}
Ben-Tal, A., Den~Hertog, D., De~Waegenaere, A., Melenberg, B., and Rennen, G.
\newblock Robust solutions of optimization problems affected by uncertain
  probabilities.
\newblock \emph{Management Science}, 59\penalty0 (2):\penalty0 341--357, 2013.

\bibitem[Blanchet \& Murthy(2019)Blanchet and Murthy]{blanchet2019quantifying}
Blanchet, J. and Murthy, K.
\newblock Quantifying distributional model risk via optimal transport.
\newblock \emph{Mathematics of Operations Research}, 44\penalty0 (2):\penalty0
  565--600, 2019.

\bibitem[Bowman et~al.(2015)Bowman, Angeli, Potts, and Manning]{snliemnlp2015}
Bowman, S.~R., Angeli, G., Potts, C., and Manning, C.~D.
\newblock A large annotated corpus for learning natural language inference.
\newblock In \emph{EMNLP}, 2015.

\bibitem[Carlini \& Wagner(2017)Carlini and Wagner]{carlini2017adversarial}
Carlini, N. and Wagner, D.
\newblock Adversarial examples are not easily detected: Bypassing ten detection
  methods.
\newblock In \emph{Proceedings of the 10th ACM workshop on artificial
  intelligence and security}, pp.\  3--14, 2017.

\bibitem[Challenge()]{yelpdataset}
Challenge, Y.~D.
\newblock data retrieved from Yelp Dataset Challenge,
  \url{https://www.yelp.com/dataset/challenge}.

\bibitem[Cisse et~al.(2017)Cisse, Bojanowski, Grave, Dauphin, and
  Usunier]{cisse2017parseval}
Cisse, M., Bojanowski, P., Grave, E., Dauphin, Y., and Usunier, N.
\newblock Parseval networks: Improving robustness to adversarial examples.
\newblock In \emph{International Conference on Machine Learning}, pp.\
  854--863. PMLR, 2017.

\bibitem[Cl{\'e}men{\c{c}}on et~al.(2008)Cl{\'e}men{\c{c}}on, Lugosi, and
  Vayatis]{clemenccon2008ranking}
Cl{\'e}men{\c{c}}on, S., Lugosi, G., and Vayatis, N.
\newblock Ranking and empirical minimization of u-statistics.
\newblock \emph{The Annals of Statistics}, 36\penalty0 (2):\penalty0 844--874,
  2008.

\bibitem[Cohen et~al.(2019)Cohen, Rosenfeld, and Kolter]{cohen2019certified}
Cohen, J., Rosenfeld, E., and Kolter, Z.
\newblock Certified adversarial robustness via randomized smoothing.
\newblock In \emph{International Conference on Machine Learning}, pp.\
  1310--1320. PMLR, 2019.

\bibitem[Cranko et~al.(2021)Cranko, Shi, Zhang, Nock, and
  Kornblith]{cranko2021generalised}
Cranko, Z., Shi, Z., Zhang, X., Nock, R., and Kornblith, S.
\newblock Generalised lipschitz regularisation equals distributional
  robustness.
\newblock In \emph{International Conference on Machine Learning}, pp.\
  2178--2188. PMLR, 2021.

\bibitem[Dai \& Van~Gool(2018)Dai and Van~Gool]{dai2018dark}
Dai, D. and Van~Gool, L.
\newblock Dark model adaptation: Semantic image segmentation from daytime to
  nighttime.
\newblock In \emph{2018 21st International Conference on Intelligent
  Transportation Systems (ITSC)}, pp.\  3819--3824. IEEE, 2018.

\bibitem[Devlin et~al.(2018)Devlin, Chang, Lee, and Toutanova]{devlin2018bert}
Devlin, J., Chang, M.-W., Lee, K., and Toutanova, K.
\newblock Bert: Pre-training of deep bidirectional transformers for language
  understanding.
\newblock \emph{arXiv preprint arXiv:1810.04805}, 2018.

\bibitem[Duchi \& Namkoong(2019)Duchi and Namkoong]{duchi2019variance}
Duchi, J. and Namkoong, H.
\newblock Variance-based regularization with convex objectives.
\newblock \emph{The Journal of Machine Learning Research}, 20\penalty0
  (1):\penalty0 2450--2504, 2019.

\bibitem[Duchi \& Namkoong(2021)Duchi and Namkoong]{duchi2021learning}
Duchi, J.~C. and Namkoong, H.
\newblock Learning models with uniform performance via distributionally robust
  optimization.
\newblock \emph{The Annals of Statistics}, 49\penalty0 (3):\penalty0
  1378--1406, 2021.

\bibitem[Duchi et~al.(2021)Duchi, Glynn, and Namkoong]{duchi2021statistics}
Duchi, J.~C., Glynn, P.~W., and Namkoong, H.
\newblock Statistics of robust optimization: A generalized empirical likelihood
  approach.
\newblock \emph{Mathematics of Operations Research}, 2021.

\bibitem[Englesson \& Azizpour(2021)Englesson and
  Azizpour]{englesson2021generalized}
Englesson, E. and Azizpour, H.
\newblock Generalized jensen-shannon divergence loss for learning with noisy
  labels.
\newblock In Beygelzimer, A., Dauphin, Y., Liang, P., and Vaughan, J.~W.
  (eds.), \emph{Advances in Neural Information Processing Systems}, 2021.
\newblock URL \url{https://openreview.net/forum?id=TiwPYwg3IRf}.

\bibitem[Gao \& Kleywegt(2016)Gao and Kleywegt]{gao2016distributionally}
Gao, R. and Kleywegt, A.~J.
\newblock Distributionally robust stochastic optimization with wasserstein
  distance.
\newblock \emph{arXiv preprint arXiv:1604.02199}, 2016.

\bibitem[Gotoh et~al.(2018)Gotoh, Kim, and Lim]{gotoh2018robust}
Gotoh, J.-y., Kim, M.~J., and Lim, A.~E.
\newblock Robust empirical optimization is almost the same as mean--variance
  optimization.
\newblock \emph{Operations research letters}, 46\penalty0 (4):\penalty0
  448--452, 2018.

\bibitem[Gulrajani \& Lopez-Paz(2021)Gulrajani and Lopez-Paz]{gulrajani2021in}
Gulrajani, I. and Lopez-Paz, D.
\newblock In search of lost domain generalization.
\newblock In \emph{International Conference on Learning Representations}, 2021.
\newblock URL \url{https://openreview.net/forum?id=lQdXeXDoWtI}.

\bibitem[Hanley \& McNeil(1982)Hanley and McNeil]{hanley1982meaning}
Hanley, J.~A. and McNeil, B.~J.
\newblock The meaning and use of the area under a receiver operating
  characteristic (roc) curve.
\newblock \emph{Radiology}, 143\penalty0 (1):\penalty0 29--36, 1982.

\bibitem[He et~al.(2020)He, Liu, Gao, and Chen]{he2020deberta}
He, P., Liu, X., Gao, J., and Chen, W.
\newblock Deberta: Decoding-enhanced bert with disentangled attention.
\newblock \emph{arXiv preprint arXiv:2006.03654}, 2020.

\bibitem[Hoeffding(1963)]{hoeffding1963probability}
Hoeffding, W.
\newblock Probability inequalities for sums of bounded random variables.
\newblock \emph{Journal of the American Statistical Association}, 58\penalty0
  (301):\penalty0 13--30, 1963.

\bibitem[Huang et~al.(2017)Huang, Liu, Van Der~Maaten, and
  Weinberger]{huang2017densely}
Huang, G., Liu, Z., Van Der~Maaten, L., and Weinberger, K.~Q.
\newblock Densely connected convolutional networks.
\newblock In \emph{Proceedings of the IEEE conference on computer vision and
  pattern recognition}, pp.\  4700--4708, 2017.

\bibitem[Husain(2020)]{husain2020distributional}
Husain, H.
\newblock Distributional robustness with ipms and links to regularization and
  gans.
\newblock \emph{arXiv preprint arXiv:2006.04349}, 2020.

\bibitem[Kim et~al.(2019)Kim, Kim, Kim, Kim, and Kim]{kim2019learning}
Kim, B., Kim, H., Kim, K., Kim, S., and Kim, J.
\newblock Learning not to learn: Training deep neural networks with biased
  data.
\newblock In \emph{Proceedings of the IEEE/CVF Conference on Computer Vision
  and Pattern Recognition}, pp.\  9012--9020, 2019.

\bibitem[Koh et~al.(2021)Koh, Sagawa, Marklund, Xie, Zhang, Balsubramani, Hu,
  Yasunaga, Phillips, Gao, et~al.]{koh2021wilds}
Koh, P.~W., Sagawa, S., Marklund, H., Xie, S.~M., Zhang, M., Balsubramani, A.,
  Hu, W., Yasunaga, M., Phillips, R.~L., Gao, I., et~al.
\newblock Wilds: A benchmark of in-the-wild distribution shifts.
\newblock In \emph{International Conference on Machine Learning}, pp.\
  5637--5664. PMLR, 2021.

\bibitem[Krizhevsky(2009)]{cifar10}
Krizhevsky, A.
\newblock Learning multiple layers of features from tiny images.
\newblock Technical report, 2009.

\bibitem[Kuhn et~al.(2019)Kuhn, Esfahani, Nguyen, and
  Shafieezadeh-Abadeh]{kuhn2019wasserstein}
Kuhn, D., Esfahani, P.~M., Nguyen, V.~A., and Shafieezadeh-Abadeh, S.
\newblock Wasserstein distributionally robust optimization: Theory and
  applications in machine learning.
\newblock In \emph{Operations Research \& Management Science in the Age of
  Analytics}, pp.\  130--166. INFORMS, 2019.

\bibitem[Lam(2016)]{lam2016robust}
Lam, H.
\newblock Robust sensitivity analysis for stochastic systems.
\newblock \emph{Mathematics of Operations Research}, 41\penalty0 (4):\penalty0
  1248--1275, 2016.

\bibitem[Lecuyer et~al.(2019)Lecuyer, Atlidakis, Geambasu, Hsu, and
  Jana]{lecuyer2019certified}
Lecuyer, M., Atlidakis, V., Geambasu, R., Hsu, D., and Jana, S.
\newblock Certified robustness to adversarial examples with differential
  privacy.
\newblock In \emph{2019 IEEE Symposium on Security and Privacy (SP)}, pp.\
  656--672. IEEE, 2019.

\bibitem[Lee \& Raginsky(2017)Lee and Raginsky]{lee2017minimax}
Lee, J. and Raginsky, M.
\newblock Minimax statistical learning with wasserstein distances.
\newblock \emph{arXiv preprint arXiv:1705.07815}, 2017.

\bibitem[Lin et~al.(2017)Lin, Feng, Santos, Yu, Xiang, Zhou, and
  Bengio]{lin2017structured}
Lin, Z., Feng, M., Santos, C. N.~d., Yu, M., Xiang, B., Zhou, B., and Bengio,
  Y.
\newblock A structured self-attentive sentence embedding.
\newblock \emph{arXiv preprint arXiv:1703.03130}, 2017.

\bibitem[Maurer \& Pontil(2009)Maurer and Pontil]{maurer2009empirical}
Maurer, A. and Pontil, M.
\newblock Empirical bernstein bounds and sample variance penalization.
\newblock In \emph{Proceedings of the Twenty Second Annual Conference on
  Computational Learning Theory}, 2009.

\bibitem[Nguyen et~al.(2007)Nguyen, Wainwright, and
  Jordan]{nguyen2007nonparametric}
Nguyen, X., Wainwright, M.~J., and Jordan, M.~I.
\newblock Nonparametric estimation of the likelihood ratio and divergence
  functionals.
\newblock In \emph{2007 IEEE International Symposium on Information Theory},
  pp.\  2016--2020. IEEE, 2007.

\bibitem[Nguyen et~al.(2010)Nguyen, Wainwright, and
  Jordan]{nguyen2010estimating}
Nguyen, X., Wainwright, M.~J., and Jordan, M.~I.
\newblock Estimating divergence functionals and the likelihood ratio by convex
  risk minimization.
\newblock \emph{IEEE Transactions on Information Theory}, 56\penalty0
  (11):\penalty0 5847--5861, 2010.

\bibitem[Russakovsky et~al.(2015)Russakovsky, Deng, Su, Krause, Satheesh, Ma,
  Huang, Karpathy, Khosla, Bernstein, et~al.]{russakovsky2015imagenet}
Russakovsky, O., Deng, J., Su, H., Krause, J., Satheesh, S., Ma, S., Huang, Z.,
  Karpathy, A., Khosla, A., Bernstein, M., et~al.
\newblock Imagenet large scale visual recognition challenge.
\newblock \emph{International journal of computer vision}, 115\penalty0
  (3):\penalty0 211--252, 2015.

\bibitem[Scarf(1958)]{scarf1958min}
Scarf, H.
\newblock A min-max solution of an inventory problem.
\newblock \emph{Studies in the mathematical theory of inventory and
  production}, 1958.

\bibitem[Shafieezadeh-Abadeh et~al.(2019)Shafieezadeh-Abadeh, Kuhn, and
  Esfahani]{shafieezadeh2019regularization}
Shafieezadeh-Abadeh, S., Kuhn, D., and Esfahani, P.~M.
\newblock Regularization via mass transportation.
\newblock \emph{Journal of Machine Learning Research}, 20\penalty0
  (103):\penalty0 1--68, 2019.

\bibitem[Sinha et~al.(2018)Sinha, Namkoong, and Duchi]{sinha2017certifying}
Sinha, A., Namkoong, H., and Duchi, J.
\newblock Certifiable distributional robustness with principled adversarial
  training.
\newblock In \emph{International Conference on Learning Representations}, 2018.
\newblock URL \url{https://openreview.net/forum?id=Hk6kPgZA-}.

\bibitem[Sreekumar et~al.(2021)Sreekumar, Zhang, and
  Goldfeld]{sreekumar2021non}
Sreekumar, S., Zhang, Z., and Goldfeld, Z.
\newblock Non-asymptotic performance guarantees for neural estimation of
  f-divergences.
\newblock In \emph{International Conference on Artificial Intelligence and
  Statistics}, pp.\  3322--3330. PMLR, 2021.

\bibitem[Staib \& Jegelka(2019)Staib and Jegelka]{staib2019distributionally}
Staib, M. and Jegelka, S.
\newblock Distributionally robust optimization and generalization in kernel
  methods.
\newblock \emph{Advances in Neural Information Processing Systems},
  32:\penalty0 9134--9144, 2019.

\bibitem[Subbaswamy et~al.(2021)Subbaswamy, Adams, and
  Saria]{subbaswamy2021evaluating}
Subbaswamy, A., Adams, R., and Saria, S.
\newblock Evaluating model robustness and stability to dataset shift.
\newblock In \emph{International Conference on Artificial Intelligence and
  Statistics}, pp.\  2611--2619. PMLR, 2021.

\bibitem[Szegedy et~al.(2014)Szegedy, Zaremba, Sutskever, Bruna, Erhan,
  Goodfellow, and Fergus]{szegedy2014intriguing}
Szegedy, C., Zaremba, W., Sutskever, I., Bruna, J., Erhan, D., Goodfellow, I.,
  and Fergus, R.
\newblock Intriguing properties of neural networks.
\newblock In \emph{International Conference on Learning Representations}, 2014.
\newblock URL \url{http://arxiv.org/abs/1312.6199}.

\bibitem[Tan \& Le(2019)Tan and Le]{tan2019efficientnet}
Tan, M. and Le, Q.
\newblock Efficientnet: Rethinking model scaling for convolutional neural
  networks.
\newblock In \emph{International Conference on Machine Learning}, pp.\
  6105--6114. PMLR, 2019.

\bibitem[Virmaux \& Scaman(2018)Virmaux and Scaman]{virmaux2018lipschitz}
Virmaux, A. and Scaman, K.
\newblock Lipschitz regularity of deep neural networks: analysis and efficient
  estimation.
\newblock In Bengio, S., Wallach, H., Larochelle, H., Grauman, K.,
  Cesa-Bianchi, N., and Garnett, R. (eds.), \emph{Advances in Neural
  Information Processing Systems}, volume~31. Curran Associates, Inc., 2018.

\bibitem[Volk et~al.(2019)Volk, M{\"u}ller, von Bernuth, Hospach, and
  Bringmann]{volk2019towards}
Volk, G., M{\"u}ller, S., von Bernuth, A., Hospach, D., and Bringmann, O.
\newblock Towards robust cnn-based object detection through augmentation with
  synthetic rain variations.
\newblock In \emph{2019 IEEE Intelligent Transportation Systems Conference
  (ITSC)}, pp.\  285--292. IEEE, 2019.

\bibitem[Weber et~al.(2021)Weber, Anand, Cervera-Lierta, Kottmann, Kyaw, Li,
  Aspuru-Guzik, Zhang, and Zhao]{weber2021toward}
Weber, M., Anand, A., Cervera-Lierta, A., Kottmann, J.~S., Kyaw, T.~H., Li, B.,
  Aspuru-Guzik, A., Zhang, C., and Zhao, Z.
\newblock Toward reliability in the nisq era: Robust interval guarantee for
  quantum measurements on approximate states.
\newblock \emph{arXiv preprint arXiv:2110.09793}, 2021.

\bibitem[Weinhold(1968)]{weinhold1968lower}
Weinhold, F.
\newblock Lower bounds to expectation values.
\newblock \emph{Journal of Physics A: General Physics}, 1\penalty0
  (3):\penalty0 305, 1968.
\newblock URL \url{https://doi.org/10.1088/0305-4470/1/3/301}.

\bibitem[Wong et~al.(2018)Wong, Schmidt, Metzen, and Kolter]{wong2018scaling}
Wong, E., Schmidt, F., Metzen, J.~H., and Kolter, J.~Z.
\newblock Scaling provable adversarial defenses.
\newblock In Bengio, S., Wallach, H., Larochelle, H., Grauman, K.,
  Cesa-Bianchi, N., and Garnett, R. (eds.), \emph{Advances in Neural
  Information Processing Systems}, volume~31. Curran Associates, Inc., 2018.

\end{thebibliography}
\bibliographystyle{icml2022}

\newpage
\appendix
\onecolumn
\section{Proofs}
\subsection{Proof of Theorem~\ref{thm:main}}
\label{apx:proof-main-thm}
We begin the proof by stating a lemma which allows one to bound inner products between elements of a Hilbert space $\cH$.
\begin{lemma}
    \label{lem:auxiliary-inner-product-bound}
    Let $\cH$ be a Hilbert space with inner product $\langle\cdot,\,\cdot\rangle$, let $A\in\cB(\cH)$ be a positive semidefinite bounded linear operator on $\cH$ and let $u,\,v\in\cH\setminus\{0\}$ be such that
    \begin{equation}
        \label{eq:overlap-condition}
        \frac{\abs{\langle u,\,v\rangle}^2}{\left(\langle u,\,u\rangle\langle v,\,v\rangle - \abs{\langle u,\,v\rangle}^2\right)} \geq \frac{\Delta}{\langle v,\,Av\rangle^2},
    \end{equation}
    where
    \begin{equation}
            \Delta := \left(\langle v,\,v\rangle\langle Av,\,Av\rangle - \langle Av,\,v\rangle^2\right).
    \end{equation}
    Then
    \begin{equation}
        \begin{gathered}
            \langle u,\,Au\rangle \geq
            \frac{\abs{\langle u,\,v\rangle}^2 \langle v,\,Av\rangle}{\langle v,\,v\rangle^2}
            - \frac{2\abs{\langle u,\,v\rangle}\sqrt{\left(\langle u,\,u\rangle\langle v,\,v\rangle - \abs{\langle u,\,v\rangle}^2\right)\Delta}}{\langle v,\,v\rangle^2}
            + \frac{\left(\langle u,\,u\rangle\langle v,\,v\rangle - \abs{\langle u,\,v\rangle}^2\right)\Delta}{\langle v,\,v\rangle^2 \langle v,\,Av\rangle}\\
        \end{gathered}
    \end{equation}
\end{lemma}
\begin{proof}[Proof of Lemma~\ref{lem:auxiliary-inner-product-bound}]
    In the following we denote by $\Re(z):=\frac{1}{2}(z + \Bar{z})$ and $\Im(z):=\frac{1}{2i}(z - \Bar{z})$ the real and imaginary parts of a complex number $z\in\C$. Let $G$ be the Gram matrix of the vectors $u,\,v,\,Av$ and recall that Gram matrices are positive semidefinite, $G \geq 0$. Since the determinant of a matrix is given by the product of its eigenvalues, it follows that
    \begin{equation}
        \label{eq:determinant-inequality}
        \begin{aligned}
            0 \leq \mathrm{det}(G) &= \begin{vmatrix}
            \langle u,\,u\rangle & \langle u,\,v\rangle & \langle u,\,Av\rangle\\
            \langle v,\,u\rangle & \langle v,\,v\rangle & \langle v,\,Av\rangle\\
            \langle Av,\,u\rangle & \langle Av,\,v\rangle & \langle Av,\,Av\rangle
            \end{vmatrix}\\
            &= \left[\left(\langle u,\,u\rangle\langle v,\,v\rangle - \abs{\langle u,\,v\rangle}^2\right)\langle Av,\,Av\rangle - \langle u,\,u\rangle\abs{\langle Av,\,v\rangle}^2\right]\\
            &\hspace{12em} + 2 \Re\left(\langle u,\,v\rangle\langle v,\,Av\rangle\langle Av,\,u\rangle\right) - \langle v,\,v\rangle \abs{\langle Av,\,u\rangle}^2.
        \end{aligned}
    \end{equation}
    Let $\phi \in \R$ be such that $e^{i\phi}\langle u,\,v\rangle = \abs{\langle u,\,v\rangle}$ and let $\Tilde{u} = e^{-i\phi}.$\footnote{We use the convention that inner products are linear in their second argument and conjugate linear in the first inner product.} Thus, we have
    \begin{equation}
        \begin{aligned}
            &\left[\left(\langle u,\,u\rangle\langle v,\,v - \abs{\langle u,\,v\rangle}^2\rangle\right)\langle Av,\,Av\rangle - \langle u,\,u\rangle\abs{\langle Av,\,v\rangle}^2 - \langle v,\,v\rangle \Im\langle(Av,\,\Tilde{u}\rangle)^2\right]\\
            &\hspace{18em}+ 2\abs{\langle u,\,v\rangle}\langle v,\,Av\rangle\Re(\langle Av,\,\Tilde{u}\rangle) - \langle v,\,v\rangle \Re(\langle Av,\,\Tilde{u}\rangle)^2 \geq 0
        \end{aligned}
    \end{equation}
    The LHS of this inequality can be seen as a quadratic polynomial in $\Re(\langle Av,\,\Tilde{u}\rangle)$ and the non-negativity effectively constrains the values that $\Re(\langle Av,\,\Tilde{u}\rangle)$ can take to be within the roots of the polynomial. Thus, we have, in particular,
    \begin{equation}
        \label{eq:main-lemma-intermediate-inequality-I}
        \begin{aligned}
            \Re(\langle Av,\,\Tilde{u}\rangle) &\geq \frac{\abs{\langle u,\,v\rangle} \langle v,\,Av\rangle}{\langle v,\,v\rangle} - \frac{\sqrt{\left(\langle u,\,u\rangle\langle v,\,v\rangle - \abs{\langle u,\,v\rangle}^2\right)\Delta - \langle v,\,v\rangle \Im(\langle Av,\,u\rangle)^2}}{\langle v,\,v\rangle}\\
            & \geq \frac{\abs{\langle u,\,v\rangle} \langle v,\,Av\rangle}{\langle v,\,v\rangle} - \frac{\sqrt{\left(\langle u,\,u\rangle\langle v,\,v\rangle - \abs{\langle u,\,v\rangle}^2\right)\Delta}}{\langle v,\,v\rangle}
        \end{aligned}
    \end{equation}
    with $\Delta := \left(\langle v,\,v\rangle\langle Av,\,Av\rangle - \langle Av,\,v\rangle^2\right)$.
    Since $A$ is positive semidefinite, it has a square root, i.e. there exists a linear operator $A^{1/2}$ with $A^{1/2}A^{1/2} = A$. It follows that
    \begin{equation}
        \begin{aligned}
            \Re(\langle Av,\,\Tilde{u}\rangle) = &\overset{(i)}{\leq} \abs{\langle Av,\,\Tilde{u}\rangle} = \abs{\langle Av,\,u\rangle} \overset{(ii)}{=} \abs{\langle A^{1/2}v,\,A^{1/2}u\rangle}\\
            & \overset{(iii)}{\leq} \sqrt{\langle A^{1/2}u,\,A^{1/2}u\rangle}\sqrt{\langle A^{1/2}v,\,A^{1/2}v\rangle}\\
            & = \sqrt{\langle u,\,Au\rangle}\sqrt{\langle v,\,Av\rangle}
        \end{aligned}
    \end{equation}
    where in $(i)$ we have used $\Re(z) \leq \abs{z}$ for any $z\in\C$, in $(ii)$ we have used that $A^{1/2}$ is self-adjoint and in $(iii)$  we have used the Cauchy-Schwarz inequality. Combining this with~\eqref{eq:main-lemma-intermediate-inequality-I} and dividing each side by $\sqrt{\langle v,\,Av\rangle}$ yields
    \begin{equation}
        \label{eq:main-lemma-intermediate-inequality-II}
        \sqrt{\langle u,\,Au\rangle} \geq \frac{\abs{\langle u,\,v\rangle} \sqrt{\langle v,\,Av\rangle}}{\langle v,\,v\rangle} - \frac{\sqrt{\left(\langle u,\,u\rangle\langle v,\,v\rangle - \abs{\langle u,\,v\rangle}^2\right)\Delta}}{\langle v,\,v\rangle\sqrt{\langle v,\,Av\rangle}}
    \end{equation}
    The RHS in this inequality is non-negative as long as
    \begin{equation}
        \abs{\langle u,\,v\rangle} \geq \frac{\sqrt{\left(\langle u,\,u\rangle\langle v,\,v\rangle - \abs{\langle u,\,v\rangle}^2\right)\Delta}}{\langle v,\,Av\rangle}.
    \end{equation}
    Thus, in this case, squaring both sides of~\eqref{eq:main-lemma-intermediate-inequality-II} yields
    \begin{equation}
        \langle u,\,Au\rangle \geq
        \frac{\abs{\langle u,\,v\rangle}^2 \langle v,\,Av\rangle}{\langle v,\,v\rangle^2}
        - 2 \frac{\abs{\langle u,\,v\rangle}\sqrt{\left(\langle u,\,u\rangle\langle v,\,v\rangle - \abs{\langle u,\,v\rangle}^2\right)\Delta}}{\langle v,\,v\rangle^2}
        + \frac{\left(\langle u,\,u\rangle\langle v,\,v\rangle - \abs{\langle u,\,v\rangle}^2\right)\Delta}{\langle v,\,v\rangle^2 \langle v,\,Av\rangle}
    \end{equation}
    which is the desired result.
\end{proof}
We will now show how Lemma~\ref{lem:auxiliary-inner-product-bound} can be used to upper bound the worst-case risk~\eqref{eq:worst-case-risk}. Let $Q\in\cP(\cZ)$ be an arbitrary probability measure on $\cZ$ with $H(P,\,Q) \leq \rho$. Denote by $\psi_P,\,\psi_Q$ the positive square roots of the Radon-Nikodyim derivatives of $P$ and $Q$, respectively, with respect to an arbitrary measure $\mu$ with $P,\,Q \ll \mu$\footnote{Such a measure $\mu$ always exists as one can choose $\mu = P + Q$.}
\begin{equation}
    \psi_P := \sqrt{\frac{dP}{d\mu}}\hspace{2em}\text{and}\hspace{2em}\psi_Q := \sqrt{\frac{dQ}{d\mu}}.
\end{equation}
Note that $\psi_P$ and $\psi_Q$ are square-integrable with respect to $\mu$ and real-valued, $\psi_P,\,\psi_Q\in L_2(\cZ,\,\Sigma,\mu)$ where we set $\Sigma$ to be the Borel $\sigma$-algebra on $\cZ$ and assume that $L_2$ contains only real-valued functions. It is well known that $L_2$ together with the inner product $\langle f,\,g\rangle_{L_2} := \int_{\cZ}fg\,d\mu$ is a Hilbert space. Furthermore, the space of essentially bounded functions $L_{\infty}(\cZ,\,\Sigma,\mu)$ is isometrically isomorphic to the set of linear bounded operators on $\cZ$. That is, each $f\in L_\infty$ defines a linear operator $M_f$ via pointwise mulitplication, $L_2 \ni \psi \mapsto M_f \psi$ with $M_f\psi: z \mapsto (M_f\psi)(z) = f(z)\cdot\psi(z)$. It follows that for any $f \in L_\infty$, we can write its expectation with respect to $P$ (and equivalently $Q$) in terms of the inner product on $L_2$
\begin{equation}
    \bE_{Z\sim P}[f(Z)] = \int_{\cZ}f(z)\,dP(z) = \int_{\cZ} f(z)\frac{dP}{d\mu}(z)\,d\mu(z) = \int_\cZ \psi_P(z)f(z)\psi_P(z)\,d\mu(z) = \langle \psi_P,\,M_f\psi_P\rangle_{L_2}
\end{equation}
Similarly, we can write the variance of $f(Z)$ with respect to $P$ (and equivalently $Q$) in terms of inner products as
\begin{equation}
    \bV_{Z\sim P}[f(Z)] = \langle \psi_P,\,M_{f^2}\psi_P\rangle_{L_2} - \langle \psi_P,\,M_{f}\psi_P\rangle_{L_2}^2.
\end{equation}
To simplify notation, we write $f \cdot \psi$ for the image of $\psi$ under $M_f$ for $f\in L_\infty$ and we drop the subscript in the inner product whenever it is clear from context. Recall that $M$ is an upper bound on the loss function $\ell$, so that $\sup_{z\in\cZ}\abs{\ell(z)} \leq M$. It follows that the function $f_\ell(\cdot) := M - \ell(\cdot)$ is essentially bounded with respect to $\mu$ and hence defines a bounded linear operator (which is also self-adjoint since we only consider real-valued functions in this work). Applying Lemma~\ref{lem:auxiliary-inner-product-bound} to the Hilbert space $L_2$ and identifying $u \equiv \psi_Q$, $v\equiv\psi_P$ and $A$ with the operator defined by $f_\ell$ immediately yields the lower bound
\begin{equation}
    \label{eq:main-proof-lemma-instantiate}
    \begin{aligned}
        \bE_Q[M - \ell(Z)] \, &\geq \, \abs{\langle \psi_P,\,\psi_Q\rangle}^2\bE_P[M - \ell(Z)]\\
        &\hspace{-2em} - 2 \abs{\langle \psi_P,\,\psi_Q\rangle}\sqrt{(1 - \abs{\langle \psi_P,\,\psi_Q\rangle}^2)\bV_P[M - \ell(Z)]} + \frac{(1 - \abs{\langle \psi_P,\,\psi_Q\rangle}^2)\bV_P[M - \ell(Z)]}{\bE_P[M - \ell(Z)]}
    \end{aligned}
\end{equation}
Rearranging terms and noting that $\bV_P[M - \ell(Z)] = \bV_P[\ell(Z)]$ leads to
\begin{equation}
    \label{eq:main-proof-before-hellinger}
    \begin{aligned}
        \bE_Q[\ell(Z)] \, &\leq \, \abs{\langle \psi_P,\,\psi_Q\rangle}^2\bE_P[\ell(Z) - M] + M\\
        &\hspace{-2em} + 2\abs{\langle \psi_P,\,\psi_Q\rangle}\sqrt{(1 - \abs{\langle \psi_P,\,\psi_Q\rangle}^2)\bV_P[\ell(Z)]} - \frac{(1 - \abs{\langle \psi_P,\,\psi_Q\rangle}^2)\bV_P[\ell(Z)]}{\bE_P[M - \ell(Z)]}.
    \end{aligned}
\end{equation}
Note that the inner product $\langle \psi_P,\,\psi_Q\rangle$ is known as the Hellinger affinity and related to the squared Hellinger distance between $P$ and $Q$ via
\begin{equation}
    \begin{aligned}
        H^2(P,\,Q) &= \frac{1}{2}\int_{\cZ}\left(\psi_P - \psi_Q\right)^2\,\,d\mu = 1 - \int_{\cZ}\psi_P\psi_Q\,d\mu = 1 - \langle \psi_P,\,\psi_Q\rangle.
    \end{aligned}
\end{equation}
Thus, the requirement that the inner product $\langle\psi_P,\,\psi_Q\rangle$ satisfies is lower bounded by the quantity in~\eqref{eq:overlap-condition} can be expressed as an upper bound on $\rho^2$ as
\begin{equation}
    \rho^2 \leq 1 - \left[ 1 + \left(\frac{M - \bE_P[\ell(Z)]}{\sqrt{\bV_P[\ell(Z)]}}\right)^2\right]^{-1/2}.
\end{equation}
Finally, substituting $\langle\psi_P,\,\psi_Q\rangle = 1 - H^2(P,\,Q) = 1 - \rho^2$ in~\eqref{eq:main-proof-before-hellinger}, setting $C_\rho = \sqrt{\rho^2(2-\rho^2)(1-\rho^2)^2}$ and rearranging terms yields
\begin{equation}
    \begin{aligned}
        \bE_Q[\ell(Z)] &\,\leq \, \bE_P[\ell(Z)] + 2C_\rho\sqrt{\bV_P[\ell(Z)]}\\
        &\hspace{4em}+ \rho^2(2-\rho^2)\left[M - \bE_P[\ell(Z)] - \frac{\bV_P[\ell(Z)]}{\bE_P[M - \ell(Z)]}\right].
    \end{aligned}
\end{equation}
Since the choice of $Q$ was arbitrary and the RHS in this inequality does not depend on $Q$, taking the supremum of the LHS over all $Q$ with $H(P,\,Q) \leq \rho$ gives the desired result.

\subsection{A lower bound version of Theorem~\ref{thm:main}}
\label{apx:proof-lower-bound}
Given the proof of Theorem~\ref{thm:main}, it is straightforward to adapt it so as to yield a \emph{lower} bound on expectation values using Lemma~\ref{lem:auxiliary-inner-product-bound}. Indeed, by instantiating this Lemma with the function $\ell$ (instead of $f_\ell(\cdot):=M - \ell(\cdot)$) we obtain a lower bound by following the analogous, subsequent reasoning as in the proof of Theorem~\ref{thm:main}.
\begin{theorem}[Lower bound]
    \label{thm:main-lower-bound}
    Let $\ell\colon\cZ \to \R_+$ be a nonnegative function taking values in $\cZ$. Then, for any probability measure $P$ on $\cZ$ and $\rho > 0$ we have
    \begin{equation}
        \begin{aligned}
            &\inf_{Q \in B_\rho(P)} \bE_{Q}[\ell(Z)] \geq \bE_{P}[\ell(Z)] - 2C_\rho\sqrt{\bV_{P}[\ell(Z)]} - \rho^2(2-\rho^2)\bigg[\bE_{P}[\ell(Z)] - \frac{\bV_{P}[\ell(Z)]}{\bE_{P}[\ell(Z)]}\bigg]
        \end{aligned}
    \end{equation}
    where $C_\rho = \sqrt{\rho^2(1-\rho^2)^2(2-\rho^2)}$ and $B_\rho(P) = \{Q \in \cP(\cZ)\colon \, H(P,\,Q) \leq \rho\}$ is the Hellinger ball of radius $\rho$ centered at $P$. The radius $\rho$ is required to be small enough such that
    \begin{equation}
        \rho^2 \leq 1 - \left[1 + \frac{\bE_{P}[\ell(Z)]^2}{\bV_{P}[\ell(Z)]}\right]^{-1/2}.
    \end{equation}
\end{theorem}

\section{Finite Sampling Errors}
\label{apx:finite-sampling}
Here we explain the reasoning behind the finite-sampling version of our main Theorem stated in Corollary~\ref{cor:finite-sampling-bound}.
Let us first recall a version of Hoeffding's inequality, formulated in terms of our setting.
\begin{theorem}[\cite{hoeffding1963probability}]
    Let $Z_1,\,\ldots,\,Z_n$ be independent random variables drawn from $P$ and taking values in $\cZ$. Let $\ell\colon\cZ \to [0,\,M]$ be a loss function and let $\hat{L}_n:=\frac{1}{n}\sum_{i=1}^n\ell(Z_i)$ be the mean under the empirical distribution $\hat{P}_n$.
    Then, for $\delta > 0$, with probability at least $1 - \delta$,
    \begin{equation}
        \bE_P[\ell(Z)] \leq \hat{L}_n + M\sqrt{\frac{\ln 1/\delta}{2n}}.
    \end{equation}
\end{theorem}
We remark that one could in principle different concentration inequalities at this stage which can potentially improve upon Hoeffding's inequality.
For example,~\cite{maurer2009empirical} present a finite sampling version of Bennett's inequality which is known to be an improvement over Hoeffding's inequality in the low variance regime.
We leave such considerations for interesting future work.
Recall that the certificate~\eqref{eq:main-generalization-bound} is monotonically increasing in the variance. For this reason, we are interested in an upper bound on the population variance which can be computed from finite samples. To achieve this, we use the variance bound presented in Theorem 10 in~\cite{maurer2009empirical} which we state here for completeness and adapt it to our use case.
\begin{theorem}[\cite{maurer2009empirical}, Theorem 10]
    Let $Z_1,\,\ldots,\,Z_n$ be independent random variables drawn from $P$ and taking values in $\cZ$. For a loss function $\ell\colon\cZ \to [0,\,M]$, let $S_n^2:=\frac{1}{n(n-1)}\sum_{1 \leq i < j \leq n}^n (\ell(Z_i) - \ell(Z_j))^2$ be the unbiased estimator of the variance of the random variable $\ell(Z)$, $Z\sim P$. Then, for $\delta > 0$, with probability at least $1 - \delta$,
    \begin{equation}
        \sqrt{\bV_P[\ell(Z)]} \leq \sqrt{S_n^2} + M\sqrt{\frac{2\ln 1/\delta}{n - 1}}
    \end{equation}
\end{theorem}
Finally, we employ the union bound to upper bound both expectation and variance simultaneously with high probability. Thus, for any $\delta > 0$, we have with probability at least $1-\delta$
\begin{equation}
    \begin{gathered}
        \bE_P[\ell(Z)] \leq \hat{L}_n + M\sqrt{\frac{\ln 2/\delta}{2n}}, \,\\
        \sqrt{\bV_P[\ell(Z)]} \leq \sqrt{S_n^2} + M\sqrt{\frac{2\ln 2/\delta}{n - 1}}.
    \end{gathered}
\end{equation}
Finally, plugging in these upper bounds for the population quantities in Theorem~\ref{thm:main} leads to the desired finite sampling bound.
Getting the finite sampling version of the lower bound in Theorem~\ref{thm:main-lower-bound} is analogous by using the corresponding lower bound variant of Hoeffding, but still the same upper bound for the variance.

\section{Synthetic Dataset}
\label{apx:synthetic-dataset}
We consider a binary classification task with covariates $X\in \R^2$ and labels $Y \in {\pm 1}$, where the data is distributed according to the Gaussian mixture
\begin{equation}
    X\lvert\,Y=y \sim \cN(y\cdot \mu,\,\Id_2).
\end{equation}
with $p(y) =\nicefrac{1}{2}$ and $\mu=(2,\,0)^T\in\R^2$. When considering the distribution shift $P \to Q$ arising from perturbations $X \mapsto X + \delta$ for a fixed $\delta\in\R^2$, both the Wasserstein distance and Hellinger distance can be evaluated as functions of the $L_2$-norm of the perturbation:
\begin{equation}
    W_2(P,\,Q) = \norm[2]{\delta},\,\,\,
    H(P,\,Q) = \sqrt{1 - e^{-\nicefrac{\norm[2]{\delta}^2}{8}}}.
\end{equation}
For our classification model, we use a small neural network with ELU activations and 2 hidden layers of size 4 and 2. The ELU activations, in combination with spectral normalization of the weights, enforce the model to be smooth and hence satisfy the assumptions required for the certificate from~\cite{sinha2017certifying}.

\section{Lipschitz Constant for Gradients of Neural Networks with Jensen-Shannon Divergence Loss}
\label{apx:wrm-comparison-lipschitz}
Let us first recall the dual reformulation of the Wasserstein worst-case risk, which is the central result that underpins the distributional robustness certificate presented in~\cite{sinha2017certifying}.
\begin{proposition}[\cite{sinha2017certifying}, Proposition 1]
    Let $\ell\colon\Theta\times \cZ \to \R$ and $c\colon\Theta\times \cZ \to \R_+$ be continuous, and let $\phi_\gamma(\theta;\,z_0):=\sup_{z\in\cZ}\{\ell(\theta;\,z) - \gamma c(z,\,z_0)\}$. Then, for any distribution $P$ and any $\rho > 0$,
    \begin{equation}
        \label{eq:wasserstein-duality}
        \sup_{Q\colon W_c(P,\,Q) \leq \rho} \bE_Q[\ell(\theta;\,Z)] = \inf_{\gamma \geq 0}\{\gamma\rho + \bE_P[\phi_\gamma(\theta;\,Z)]\}.
    \end{equation}
    where $W_c(P,\,Q) := \inf_{\pi \in \Pi(P,\,Q)}\int_{\cZ}c(z,\,z')\,d\pi(z,\,z')$ is the 1-Wasserstein distance between $P$ and $Q$.
\end{proposition}
From this result,~\cite{sinha2017certifying} derive a robustness certificate which can be instantiated to hold uniformly over a function of families parametrized by $\theta\in\Theta$, but also a certificate that holds pointwise, that is, for a single model $\ell(\theta_0;\,\cdot)$. One requirement for this certificate to be tractable is that the surrogate function $\phi_\gamma$ be concave in $z$. As shown in~\cite{sinha2017certifying}, this is the case when $\gamma$ is larger than the Lipschitz constant $L$ of the gradient of $\ell$ with respect to $z$.
Thus one needs to compute $L$ and choose $\gamma \geq L$ so that the inner maximization in~\eqref{eq:wasserstein-duality} is guaranteed to converge and hence a robustness certificate can be calculated.

Here, we present the calculation of the Lipschitz constant for the gradient of the Jensen-Shannon divergence loss with respect to input features. For the remainder of this section, we set $\cZ = \cX \times \cY$ with a binary label space $\abs{\cY} = C = 2$.
We will always write vectors in bold roman letters, for example $\pvec = (p_1,\,\ldots,\,p_C)\in\R^C$ and $\evec_y\in\R^C$ denotes a standard basis vector with zeros everywhere except $1$ at position $y$.
We consider a feedforward neural network with $L$ layers and ELU activation functions, denoted by $\sigma$:
\begin{equation}
    F_L(\theta;\,x) := \sigma(\theta_l\cdot\sigma_{L-1}(\theta_{L-1} \cdots \sigma(\theta_1\cdot x) \cdots ))
\end{equation}
and we are interested in calculating $L>0$ such that
\begin{equation}
    \norm[*]{\nabla \ell(F_L(\theta;\,x),\,y) - \nabla \ell(F_L(\theta;\,x'),\,y)} \leq L \norm[2]{x - x'}
\end{equation}
where the gradient is taken with respect to $x$ and where $\ell$ is the Jensen-Shannon divergence. To achieve this, we apply Proposition 5 in~\cite{sinha2017certifying} which states that the Jacobian of $F_L$ is $\beta_L(\theta)$-Lipschitz with respect to the operator norm induced by $\norm[2]{\cdot}$ with
\begin{equation}
  \beta_l(\theta) = \alpha_l(\theta)\sum_{j=1}^l\left\{\frac{L_j^1}{(L_j^0)^2}\alpha_j(\theta)\right\},
  \hspace{3em}
  \alpha_l(\theta) = \prod_{j=1}^{l}L_j^0\norm[op]{\theta_j}.
\end{equation}
where $L^0_j$ is the Lipschitz constant of each activation function $\sigma_j$ and $L_j^1$ is the Lipschitz constant of its Jacobian. It is useful to write this recursively as
\begin{equation}
  \begin{aligned}
    \alpha_{l+1}(\theta) &= L^0_{l+1} \norm[op]{\theta_{l+1}}\alpha_l(\theta),\\
    \beta_{l+1}(\theta) &= L_{l+1}^0 \norm[op]{\theta_{l+1}}\beta_l(\theta) + L^1_{l+1}\norm[op]{\theta_{l+1}}^2\alpha_l(\theta)^2
  \end{aligned}
\end{equation}
In our case, since we have ELU activations, we have $L_j^0 = L_j^1 = 1$ for all $j$ (\cite{sinha2017certifying}, example 3). Finally, viewing $\ell(\pvec,\,\evec_y)$ as an $L + 1$ layer neural network with a single output dimension, we have that $\nabla_z\ell(\pvec(x),\,y)$ is $L^*$-Lipschitz continuous with constant
\begin{equation}
  L^* = L^0_{L+1}\beta_L(\theta) + L_{L+1}^1\alpha_L(\theta)^2
\end{equation}
where we have used that $\norm[op]{\theta_{L+1}} = \norm[op]{\Id} = 1$ and where $L^0_{L+1}$ is the Lipschitz constant of the function $z \mapsto \ell(\pvec(z),\,y)$ and $L^1_{L+1}$ is the Lipschitz constant of $z \mapsto \nabla_z\ell(\pvec(z),\,y)$ and $\pvec(z)$ is the softmax probability vector
\begin{equation}
  \pvec(z) = \left(\frac{e^{z_1}}{\sum_j e^{z_j}},\,\ldots,\,\frac{e^{z_C}}{\sum_j e^{z_j}}\right)^T \in\R^C.
\end{equation}
We now show the calculation of $L^0_{L+1}$ and $L^1_{L+1}$.
Fix $z\in\R^C$ and $y\in\cY$, and let $\evec_y\R^K$ be the one hot encoded label vector with zero everywhere except at position $y$.
The Jensen-Shannon divergence loss between a vector of predicted class probabilities $\pvec$ and the class label $\evec_y$ is given by
\begin{equation}
    \ell(\pvec,\,\evec_y)= \frac{1}{2}\left(D_{KL}(\pvec\|\,\mathbf{m}) + D_{KL}(\evec_y\|\,\mathbf{m})\right)
\end{equation}
with $\mathbf{m} = \frac{1}{2}(\pvec + \evec_y)$.
The Kullback Leibler divergences are
\begin{equation}
  \begin{aligned}
    D_{KL}(\pvec\|\,\mathbf{m}) &= 1 + p_y\log\left(\frac{p_y}{1 + p_y}\right)\\
    D_{KL}(\evec_y\|\,\mathbf{m}) &= 1 + \log\left(\frac{1}{1 + p_y}\right)
  \end{aligned}
\end{equation}
where $\log = \log_2$ is the logarithm with base $2$.
The Jensen-Shannon divergence loss is thus given by
\begin{equation}
  \ell(\pvec,\,\evec_y) = 1 + \frac{1}{2}\left(p_y\log(p_y) - (1 + p_y)\log(1+p_y)\right).
\end{equation}
The gradient $\nabla_z \ell(\phat,\,e_y)$ of the loss with respect to the input $x$ is given by
\begin{equation}
  \begin{aligned}
    \nabla_z\ell(\pvec,\,\evec_y) &= \frac{1}{2}\nabla_z\left(p_y\log(p_y) - (1+p_y)\log(1+p_y)\right)\\
    &= \frac{1}{2}\nabla_z\left(p_y\log(p_y)\right) - \frac{1}{2}\nabla_z\left((1+p_y)\log(1+p_y)\right)\\
    &=\frac{1}{2}(1 + \log(p_y))\nabla_z p_y - \frac{1}{2}(1 + \log(1+p_y))\nabla_z p_y
  \end{aligned}
\end{equation}
Noting that
\begin{equation}
  \nabla_z p_y(x) = p_y (\evec_y - \pvec)
\end{equation}
yields the expression
\begin{equation}
  \nabla_z\ell(\pvec,\,\evec_y) = \frac{1}{2}\log\left(\frac{p_y}{1+p_y}\right)p_y(\evec_y - \pvec).
\end{equation}
Thus,
\begin{equation}
  \begin{aligned}
    L^0_{L+1} &= \sup_z\norm[2]{\nabla_z\ell(\pvec,\,\evec_y)} = \frac{1}{2}\sup_z\left(\log\left(\frac{1+p_y}{p_y}\right)p_y\norm[2]{\evec_y - \pvec}\right)\\
    &= \sup_z\norm[2]{\nabla_z\ell(\pvec,\,\evec_y)} = \frac{1}{\sqrt{2}}\sup_z\left(\log\left(\frac{1+p_y}{p_y}\right)p_y(1-p_y)\right) \approx 0.314568.
  \end{aligned}
\end{equation}
We will now calculate $L^1_{L+1} = \sup_x\norm[2]{J}$ where $J\equiv J_{\ell(\pvec,\evec_y)}$, is the Jacobian of $\ell(\pvec,\,\evec_y)$ and $\norm[2]{J}$ is given by the largest singular value of $J$. For ease of notation, let $f_i(x) \equiv (\nabla_z\ell)_i$ and recall that $J$ is defined by
\begin{equation}
  J = \begin{pmatrix}
    \nabla_z^T f_1\\
    \vdots\\
    \nabla_z^T f_C
\end{pmatrix}.
\end{equation}
Note that
\begin{equation}
  \begin{aligned}
    \nabla_z f_i &= \nabla_z\frac{1}{2}\log\left(\frac{p_y}{1+p_y}\right)p_y(\delta_{iy} - p_i) \\
    &= \frac{1}{2}\left(\frac{1+p_y}{p_y}\left[\frac{\nabla_z p_y}{1 + p_y} - \frac{p_y}{(1+p_y)^2}\nabla_z p_y\right]\right)p_y(\delta_{iy} - p_i) + \\
    &\hspace{3em} + \frac{1}{2}\log\left(\frac{p_y}{1+p_y}\right)(\delta_{iy} - p_i)\nabla_z p_y - \frac{1}{2}\log\left(\frac{p_y}{1+p_y}\right)p_y\nabla_z p_i\\
    &= \frac{1}{2}\left(\frac{1}{1+p_y} + \log\left(\frac{p_y}{1+p_y}\right)\right)(\delta_{iy} - p_i)\nabla_z p_y -\frac{1}{2}\log\left(\frac{p_y}{1+p_y}\right)p_y\nabla_z p_i\\
  \end{aligned}
\end{equation}
and hence, using $\nabla_z p_y = p_y(\evec_y - \pvec)$,
\begin{equation}
  \nabla_z f_i = \frac{1}{2}\left(\frac{p_y}{1+p_y} + p_y\log\left(\frac{p_y}{1+p_y}\right)\right)(\delta_{iy} - p_i)(\evec_y - \pvec) - \frac{1}{2}p_y\log\left(\frac{p_y}{1+p_y}\right)p_i(\evec_i - \pvec)
\end{equation}
It follows that the Jacobian is given by
\begin{equation}
  J = \frac{1}{2}\left(\frac{p_y}{1+p_y} + p_y\log\left(\frac{p_y}{1+p_y}\right)\right) (\evec_y - \pvec)\cdot(\evec_y - \pvec)^T + \frac{1}{2}p_y\log\left(\frac{1+p_y}{p_y}\right)(\mathrm{diag}(\pvec) - \pvec\cdot\pvec^T).
\end{equation}
Since we are only interested in the binary case $C=2$, we see that
\begin{equation}
  A := (\evec_y - \pvec)\cdot(\evec_y - \pvec)^T = (1-p_y)^2\begin{pmatrix}
    1 & -1\\
    -1 & 1
    \end{pmatrix}
\end{equation}
with spectrum $\sigma(A) = \{0, 2(1-p_y)^2\}$. The eigenvalues of $\mathrm{diag}(\pvec)$ are $p_i$ and hence $\lambda(\mathrm{diag}(\pvec)) \subseteq [0, 1]$, and $\sigma(\pvec\cdot\pvec^T) = \{0, \norm[2]{\pvec}^2\}$. It follows that $\sigma(\mathrm{diag}(\pvec) - \pvec\cdot\pvec^T) \subseteq [-\norm[2]{\pvec}^2,\,1]$. Thus, by Weyl's inequality and noting that the term in front of $(\evec_y - \pvec)\cdot(\evec_y - \pvec)^T$ is always negative, we have for any eigenvalue $\lambda$ of $J$ that
\begin{equation}
  (1-p_y)^2\left(\frac{p_y}{1+p_y} + p_y\log\left(\frac{p_y}{1+p_y}\right)\right) - \frac{1}{2}p_y\log\left(\frac{1+p_y}{p_y}\right)\norm[2]{\pvec}^2 \leq \lambda
  \leq \frac{1}{2}p_y\log\left(\frac{1+p_y}{p_y}\right)
\end{equation}
Note that $J$ is symmetric, and hence its largest singular value is given by the largest absolute value of its eigenvalues. Taking the infimum (supremum) of the LHS (RHS) with respect to $z$ yields the bounds
\begin{equation}
  -\frac{1}{2} \leq \lambda \leq \frac{1}{2}
\end{equation}
and hence
\begin{equation}
  L^1_{L+1} = \sup_z\norm[2]{J} \leq \frac{1}{2}.
\end{equation}
It follows that $\nabla_x\ell(\pvec(F_L(\theta;\,x)),\,y)$ is $L^*$-Lipschitz with
\begin{equation}
  L^* = L^0_{L+1}\beta_L(\theta) + \frac{1}{2}\alpha_L(\theta)^2
\end{equation}
and $L^0_{L+1} = 0.314568$. Finally, choosing $\gamma \geq L^*$ in~\eqref{eq:wasserstein-duality} makes the objective in the surrogate loss $\phi_\gamma$ concave and hence enables the certificate
\begin{equation}
    \begin{aligned}
        \sup_{Q\colon W_c(P,\,Q) \leq \rho} \bE_Q[\ell(\theta;\,Z)] &\leq \gamma\rho + \bE_P[\phi_\gamma(\theta;\,Z)]\\
        &= \gamma\rho + \bE_{(X,\,Y)\sim P}[\sup_{x\in\cX}\ell(F_L(\theta;\,x),\,Y) - \gamma\norm[2]{x - X}^2].
    \end{aligned}
\end{equation}

\section{Hellinger distance for mixtures of distributions with disjoint support}
\label{apx:hellinger-distance-mixture}
Consider two joint (feature, label)-distributions $P,\,Q \in\cP(\cX\times\cY)$ with densities $f_P$ and $f_Q$ with respect to a suitable measure. $P$ and $Q$ have disjoint support if
\begin{equation}
    \forall\,x\in\cX,\,y\in\cY\colon \hspace{2em} f_Q(x,\,y) > 0 \iff f_P(x,\,y) = 0.
\end{equation}
In this case, for $\gamma\in(0,\,1)$, we define the mixture measure as $\Pi_\gamma := \gamma P + (1-\gamma)Q$ with density
\begin{equation}
    \pi_\gamma(x,\,y) = \gamma f_P(x,\,y) + (1-\gamma)f_Q(x,\,y).
\end{equation}
We can calculate the squared Hellinger distance between $P$ and $\Pi_\gamma$ as
\begin{equation}
    \begin{aligned}
        H^2(P,\,\Pi_\gamma) &= 1 - \int\int_{\cX\times\cY}\sqrt{f_P(x,\,y)}\sqrt{\gamma f_P(x,\,y) + (1-\gamma)f_Q(x,\,y)}\,dx\,dy\\
        &=1 - \sqrt{\gamma}\int\int_{f_p > 0}f_P(x,\,y)\sqrt{1 +  \frac{1-\gamma}{\gamma}\frac{f_Q(x,\,y)}{f_P(x,\,y)}}\,dx\,dy\\
        &=1 - \sqrt{\gamma}\int\int_{f_p > 0}f_P(x,\,y)\,dx\,dy\\
        &= 1 - \sqrt{\gamma}.
    \end{aligned}
\end{equation}

\newpage
\section{Additional Experimental Results}
\label{apx:additional-experimental-results}
\subsection{Additional Model Architectures on CIFAR-10}
\label{apx:additional-models}
Here, we present results for a diverse set of model architectures, evaluated on the CIFAR-10 dataset.
\begin{figure}[h]
\vskip 0.2in
\centering
     \begin{subfigure}[b]{0.45\columnwidth}
         \centering
         \includegraphics[width=\columnwidth]{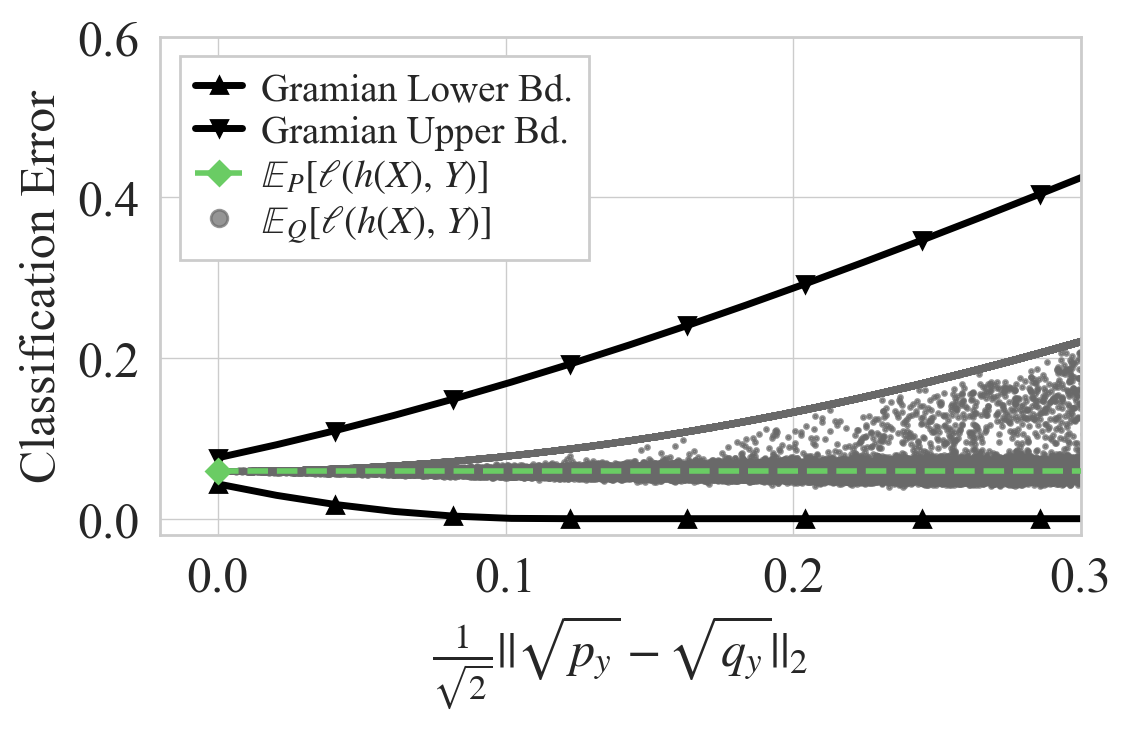}
         \caption{DenseNet-169}
     \end{subfigure}%
     \hfill
     \begin{subfigure}[b]{0.45\columnwidth}
         \centering
         \includegraphics[width=\columnwidth]{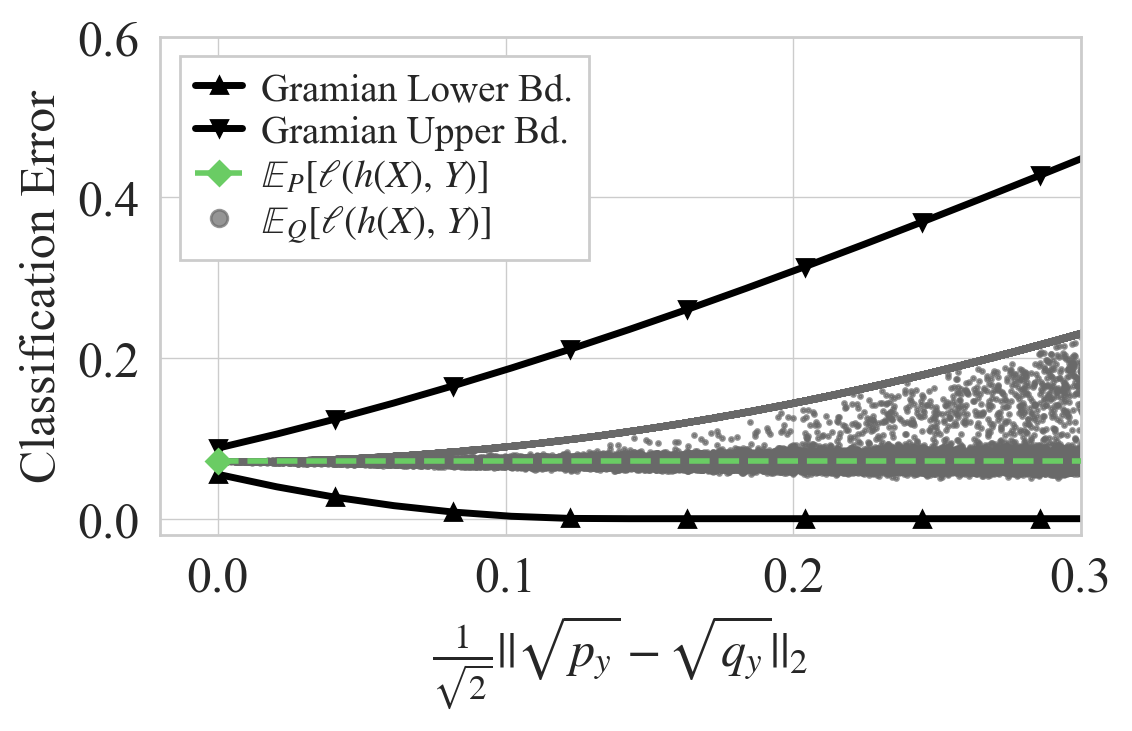}
         \caption{GoogleNet}
     \end{subfigure}
     \begin{subfigure}[b]{0.45\columnwidth}
         \centering
         \includegraphics[width=\columnwidth]{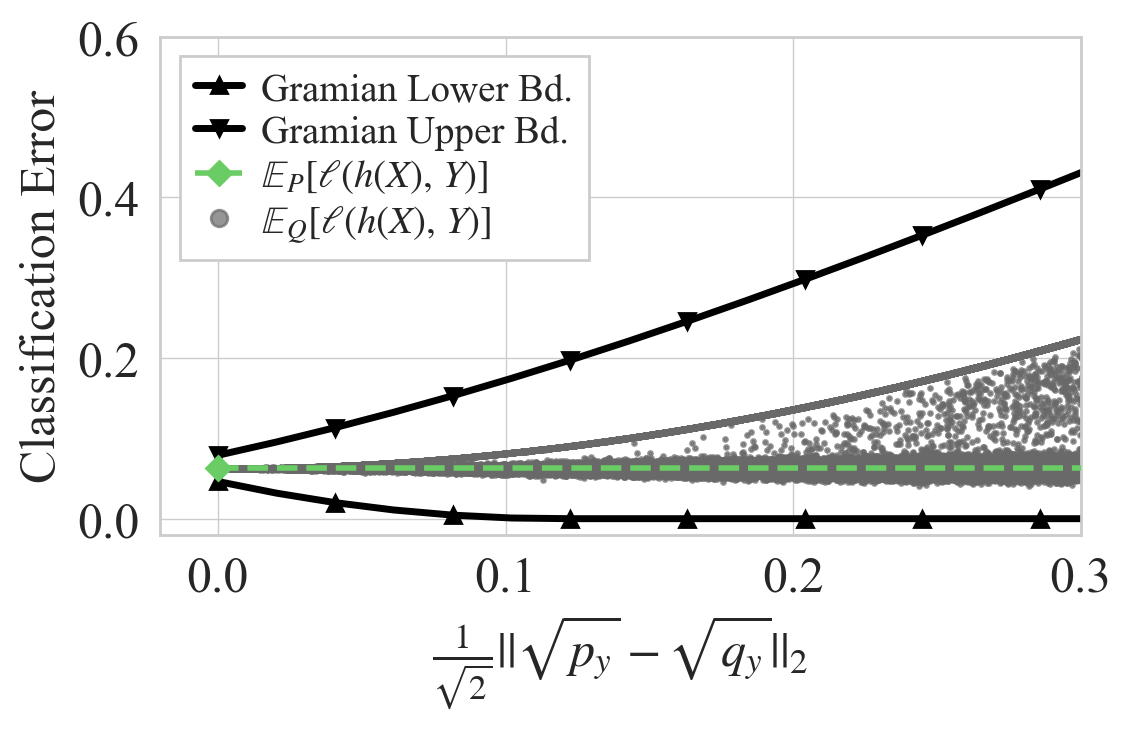}
         \caption{Inception-V3}
     \end{subfigure}%
     \hfill
     \begin{subfigure}[b]{0.45\columnwidth}
         \centering
         \includegraphics[width=\columnwidth]{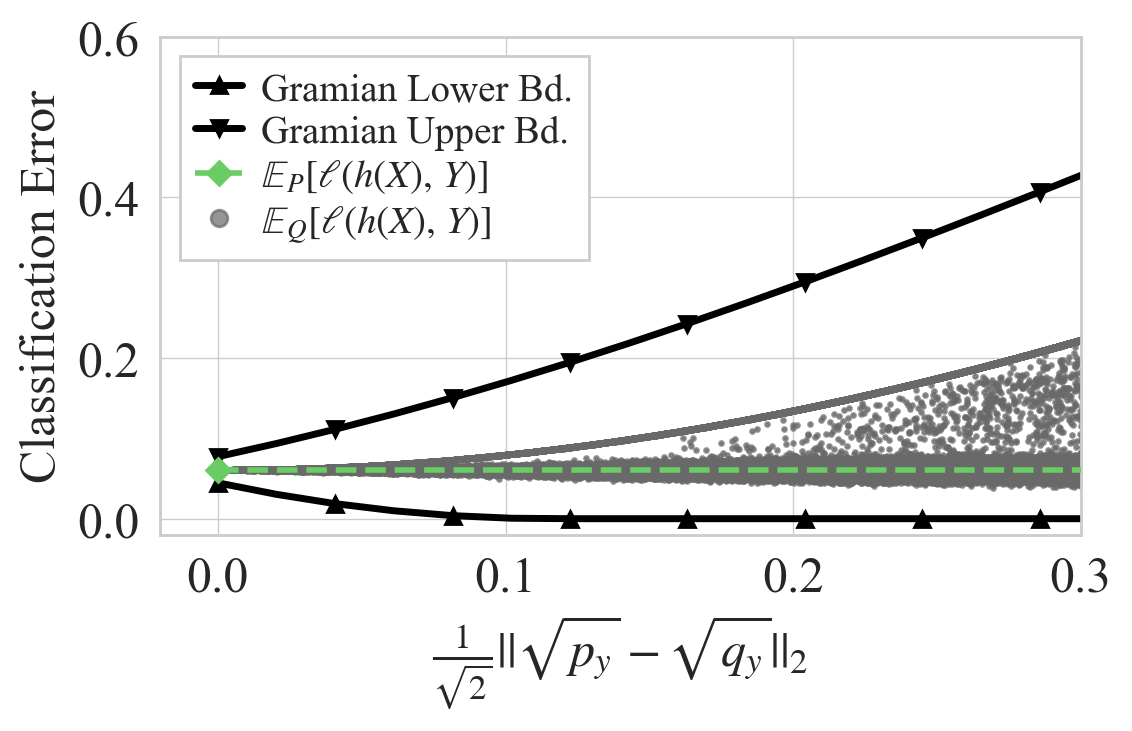}
         \caption{MobileNet-V2}
     \end{subfigure}
     \begin{subfigure}[b]{0.45\columnwidth}
         \centering
         \includegraphics[width=\columnwidth]{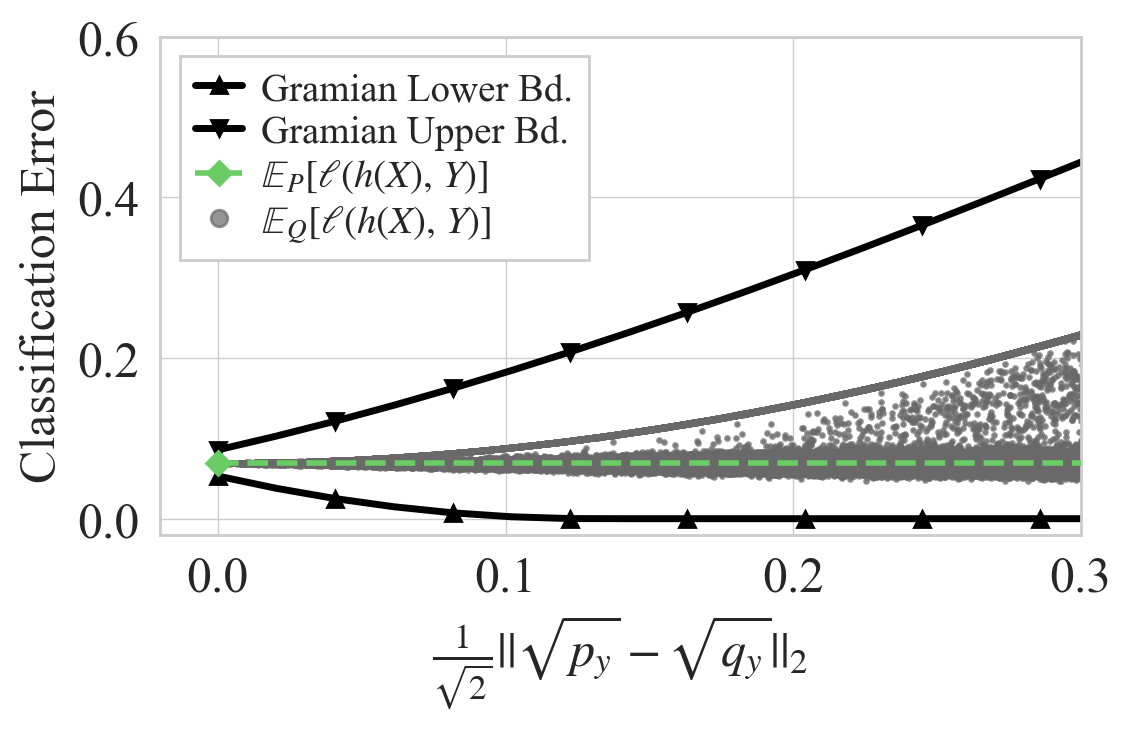}
         \caption{ResNet-18}
     \end{subfigure}%
     \hfill
     \begin{subfigure}[b]{0.45\columnwidth}
         \centering
         \includegraphics[width=\columnwidth]{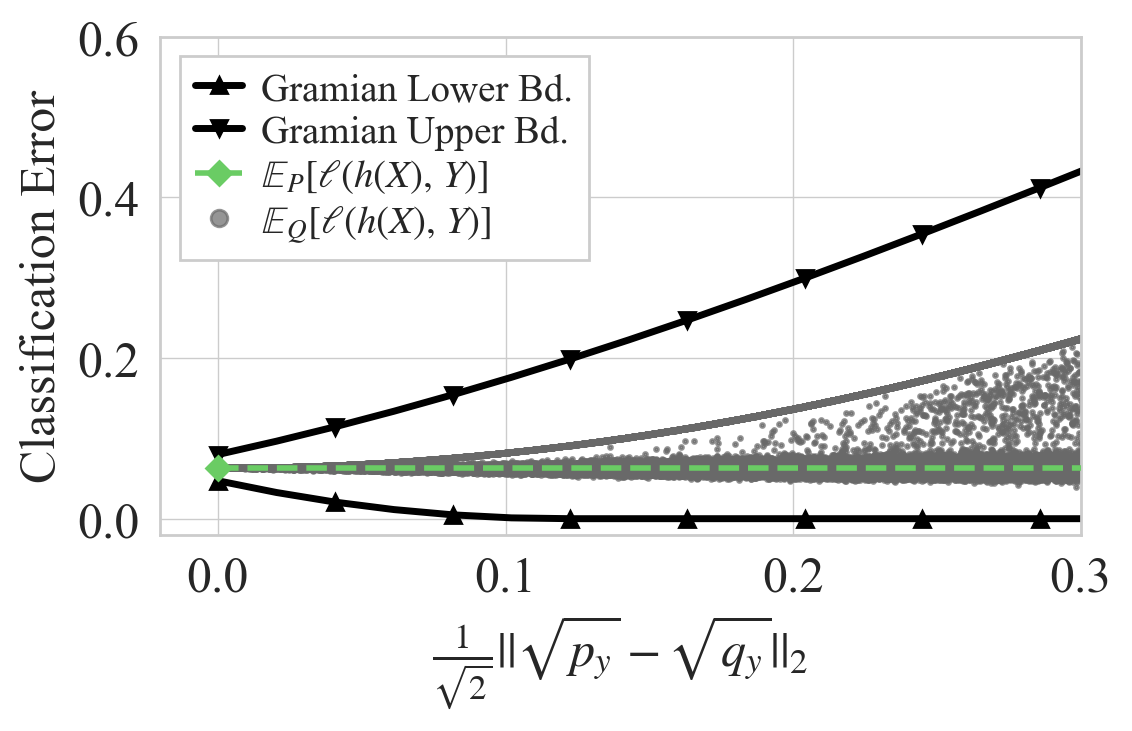}
         \caption{ResNet-50}
     \end{subfigure}
\caption{Certified classification error with label distribution shifts on CIFAR-10.}
\vskip -0.2in
\end{figure}

\begin{figure}[ht]
\vskip 0.2in
\centering
    \begin{subfigure}[b]{0.45\columnwidth}
         \centering
         \includegraphics[width=\columnwidth]{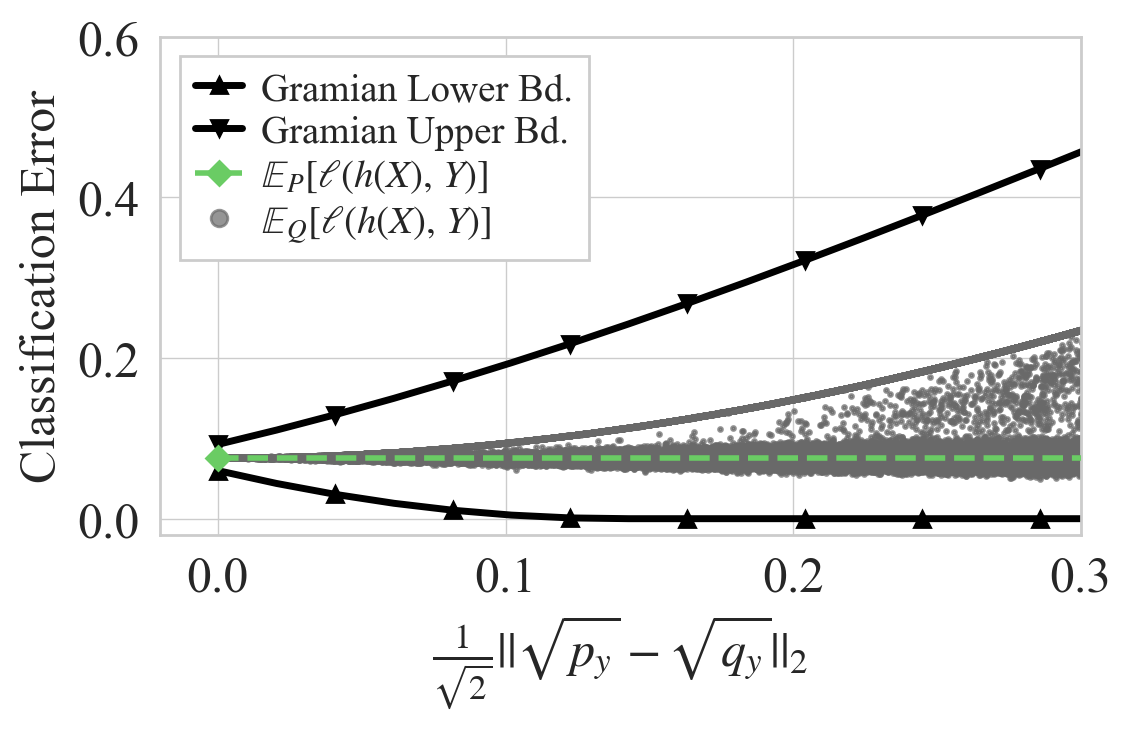}
         \caption{VGG11-BN}
     \end{subfigure}%
     \hfill
     \begin{subfigure}[b]{0.45\columnwidth}
         \centering
         \includegraphics[width=\columnwidth]{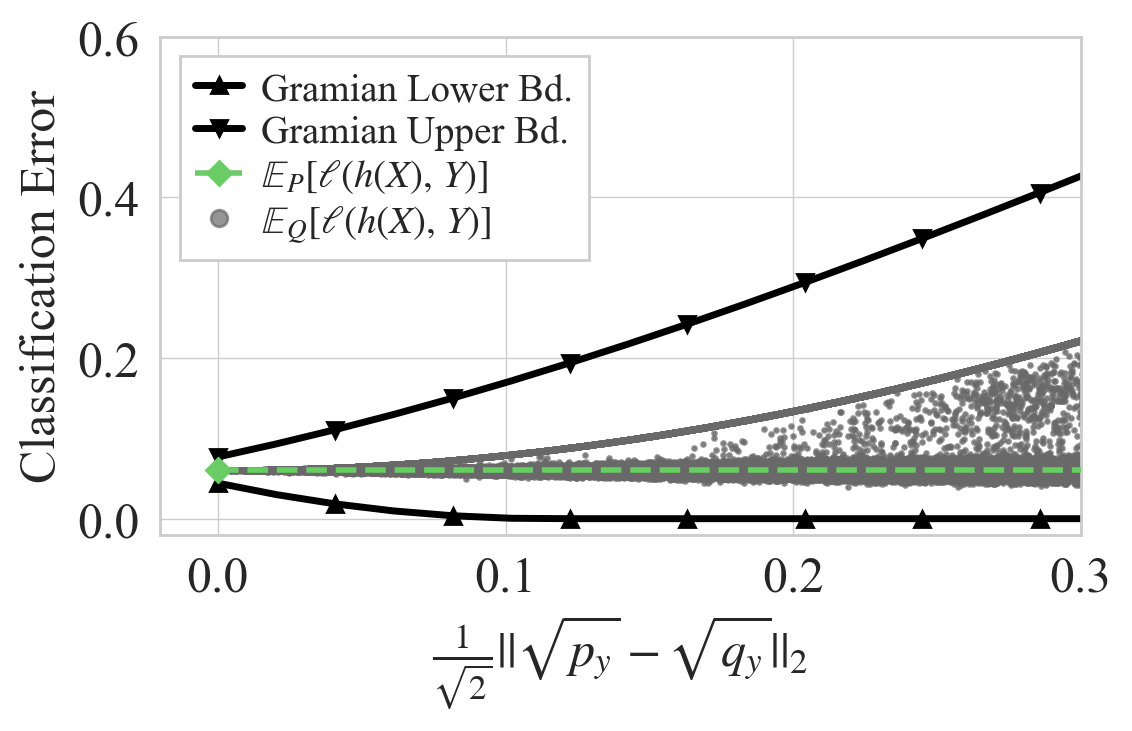}
         \caption{VGG19-BN}
     \end{subfigure}
\caption{Certified classification error with label distribution shifts on CIFAR-10.}
\vskip -0.2in
\end{figure}

\newpage
\subsection{Results for Additional Loss and Score Functions}
\label{apx:additional-functions}
Here we present additional results for JSD loss, classification error and AUC score.
\begin{figure}[h]
\vskip 0.2in
\centering
     \begin{subfigure}[b]{0.45\columnwidth}
         \centering
         \includegraphics[width=\columnwidth]{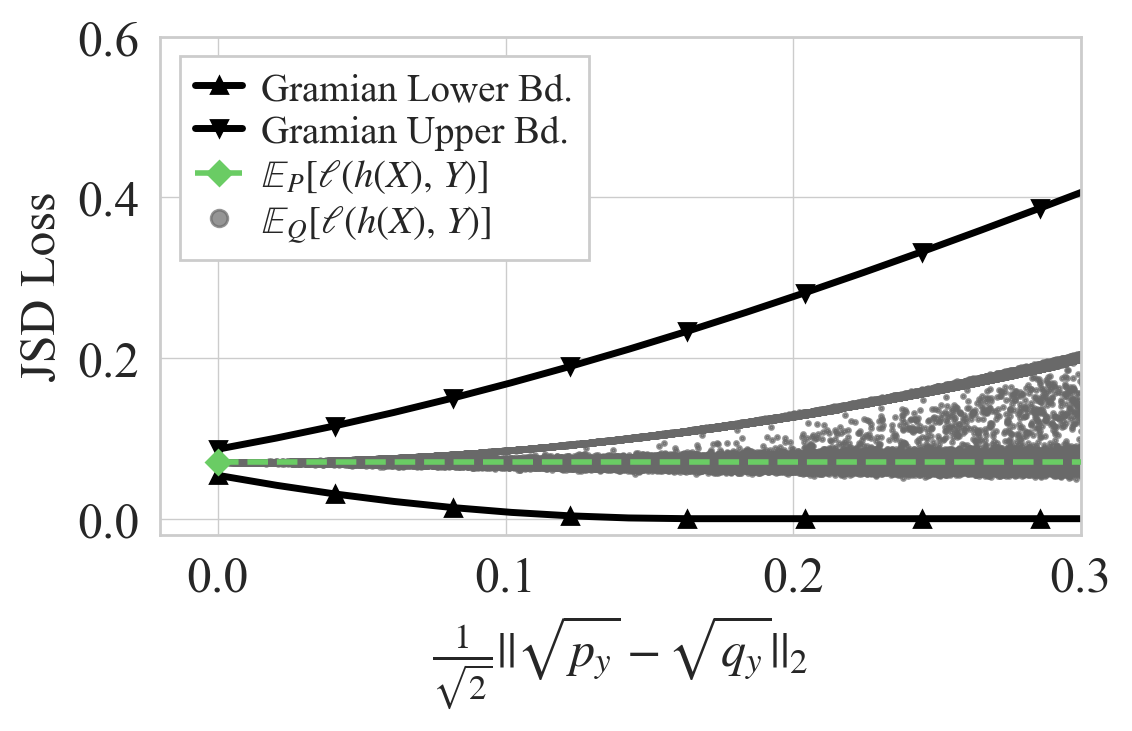}
         \caption{DenseNet-121 on CIFAR-10}
     \end{subfigure}%
     \hfill
     \begin{subfigure}[b]{0.45\columnwidth}
         \centering
         \includegraphics[width=\columnwidth]{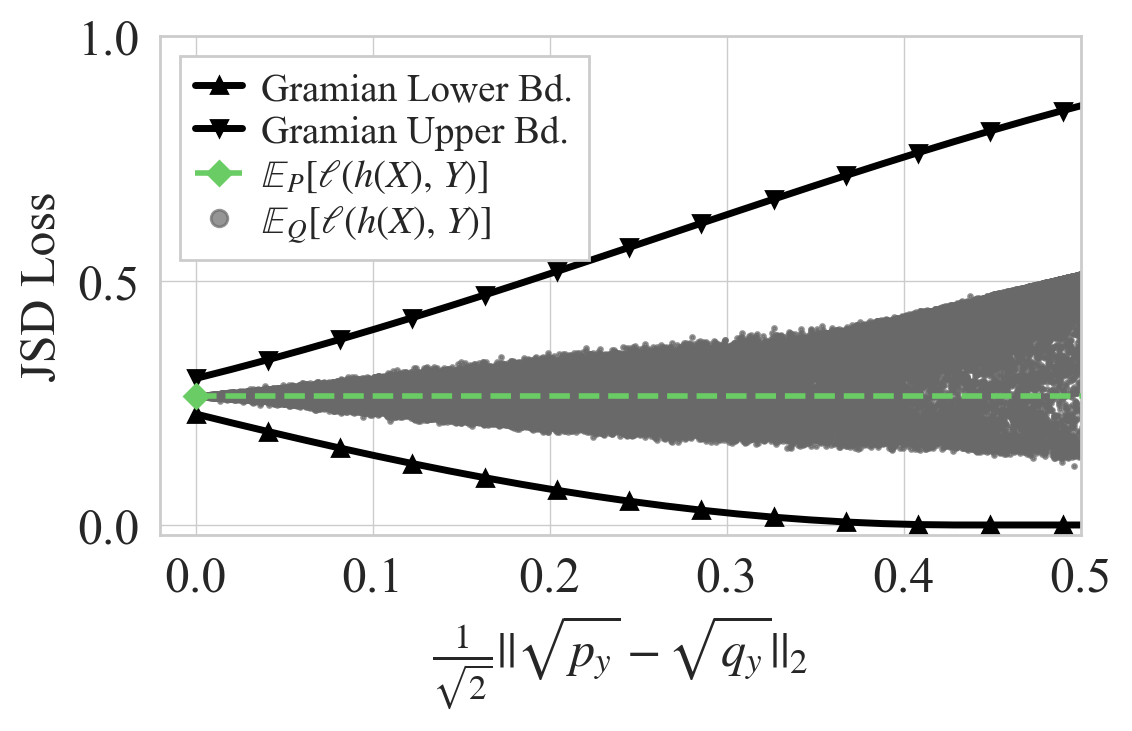}
         \caption{BERT on Yelp}
     \end{subfigure}
\caption{Certified JSD Loss with label distribution shifts.}
\vskip -0.2in
\end{figure}

\begin{figure}[h]
\vskip 0.2in
\centering
     \begin{subfigure}[b]{0.45\columnwidth}
         \centering
         \includegraphics[width=\columnwidth]{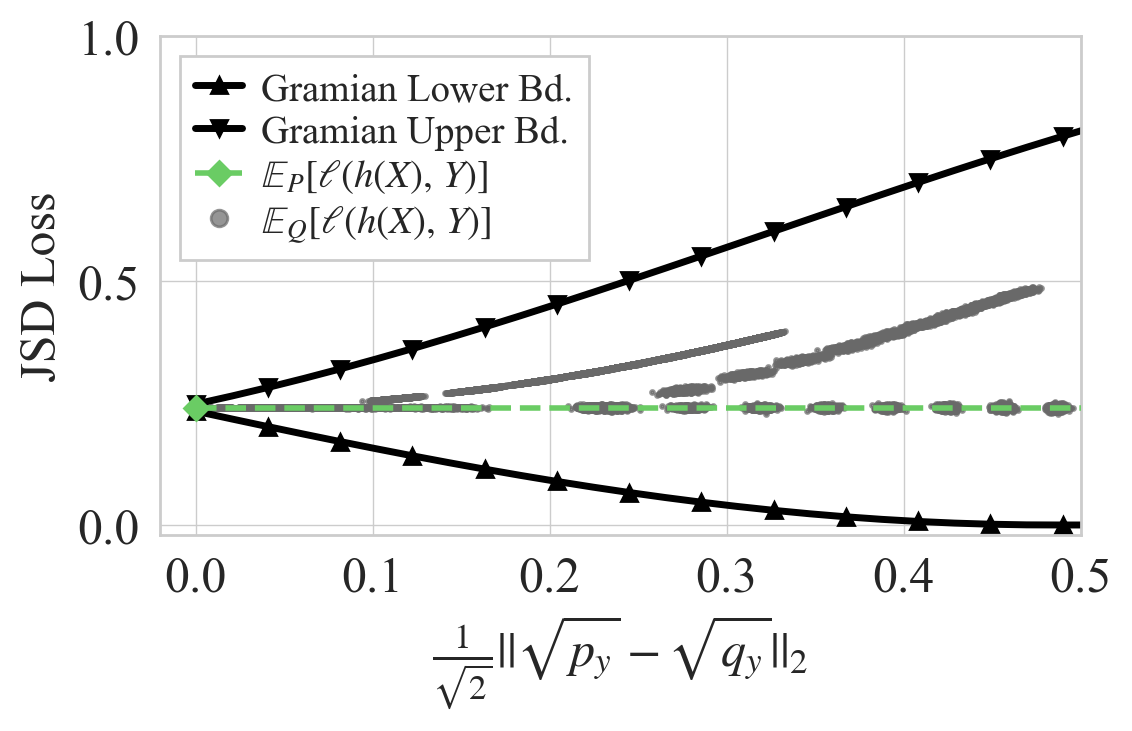}
         \caption{EfficientNet-B7 on ImageNet-1k}
     \end{subfigure}%
     \hfill
     \begin{subfigure}[b]{0.45\columnwidth}
         \centering
         \includegraphics[width=\columnwidth]{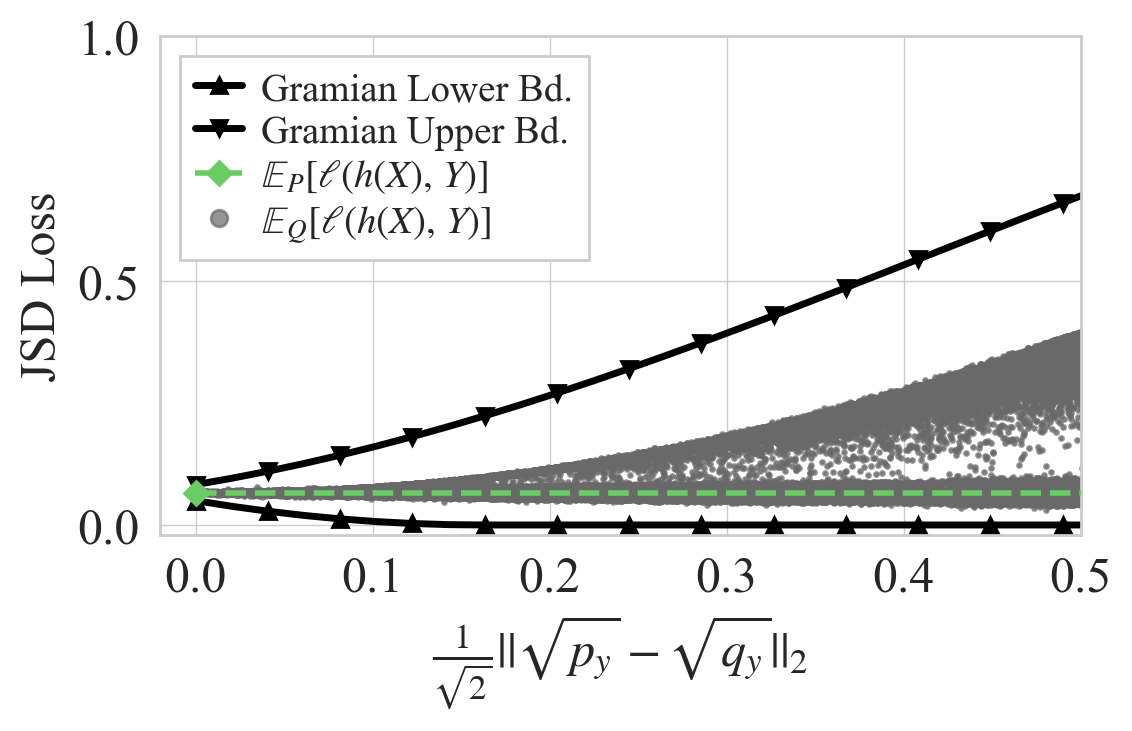}
         \caption{DeBERTa on SNLI}
     \end{subfigure}
\caption{Certified JSD Loss with label distribution shifts.}
\vskip -0.2in
\end{figure}

\begin{figure}[h]
\vskip 0.2in
\centering
     \begin{subfigure}[b]{0.45\columnwidth}
         \centering
         \includegraphics[width=\columnwidth]{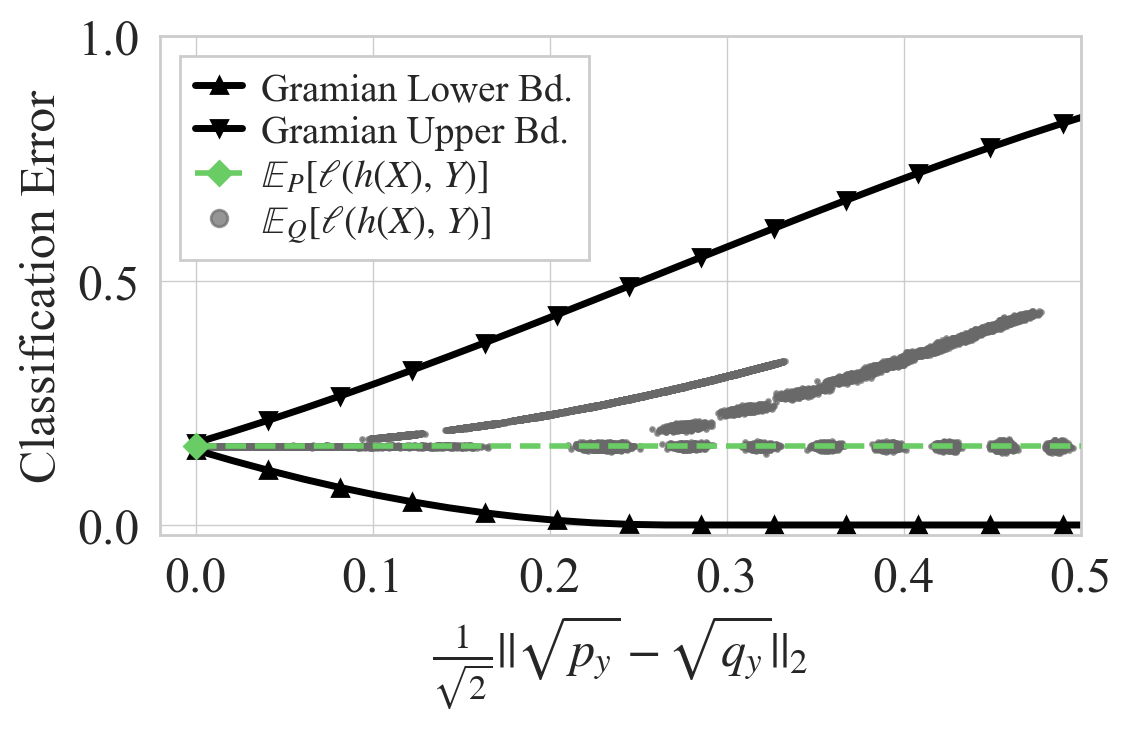}
         \caption{EfficientNet-B7 on ImageNet-1k}
     \end{subfigure}%
     \hfill
     \begin{subfigure}[b]{0.45\columnwidth}
         \centering
         \includegraphics[width=\columnwidth]{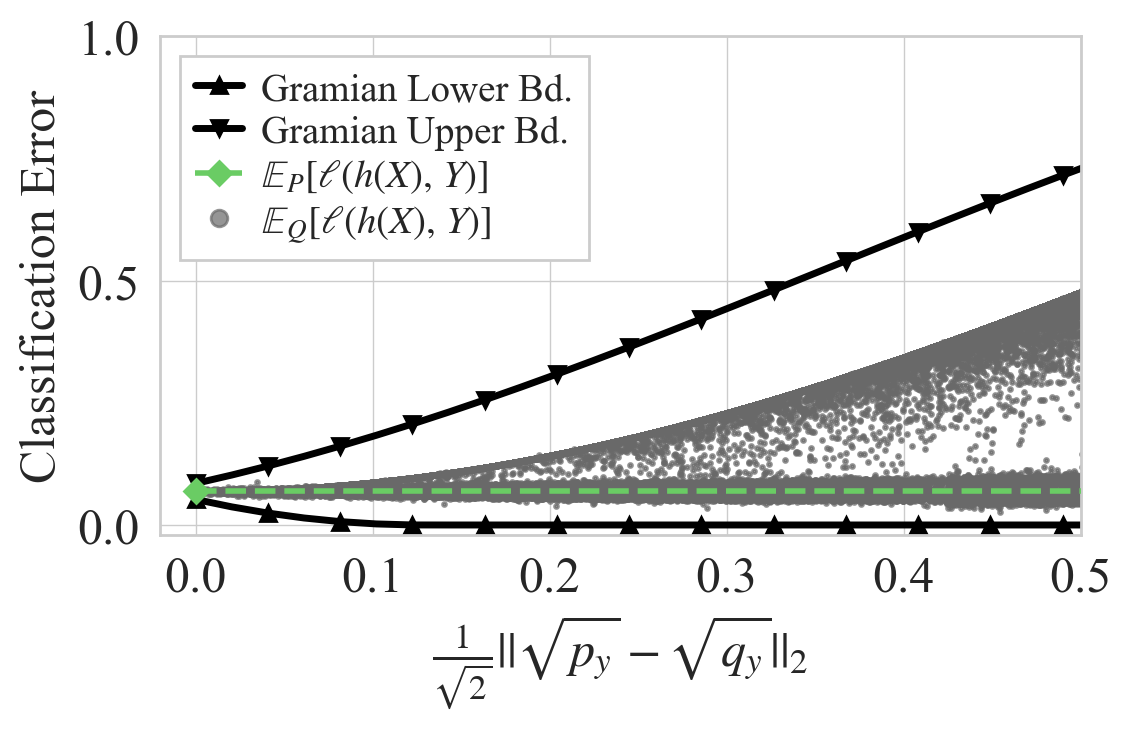}
         \caption{DeBERTa on SNLI}
     \end{subfigure}
    \caption{Certified classification error with label distribution shifts.}
\vskip -0.2in
\end{figure}

\begin{figure}[h]
\vskip 0.2in
\centering
     \includegraphics[width=.45\columnwidth]{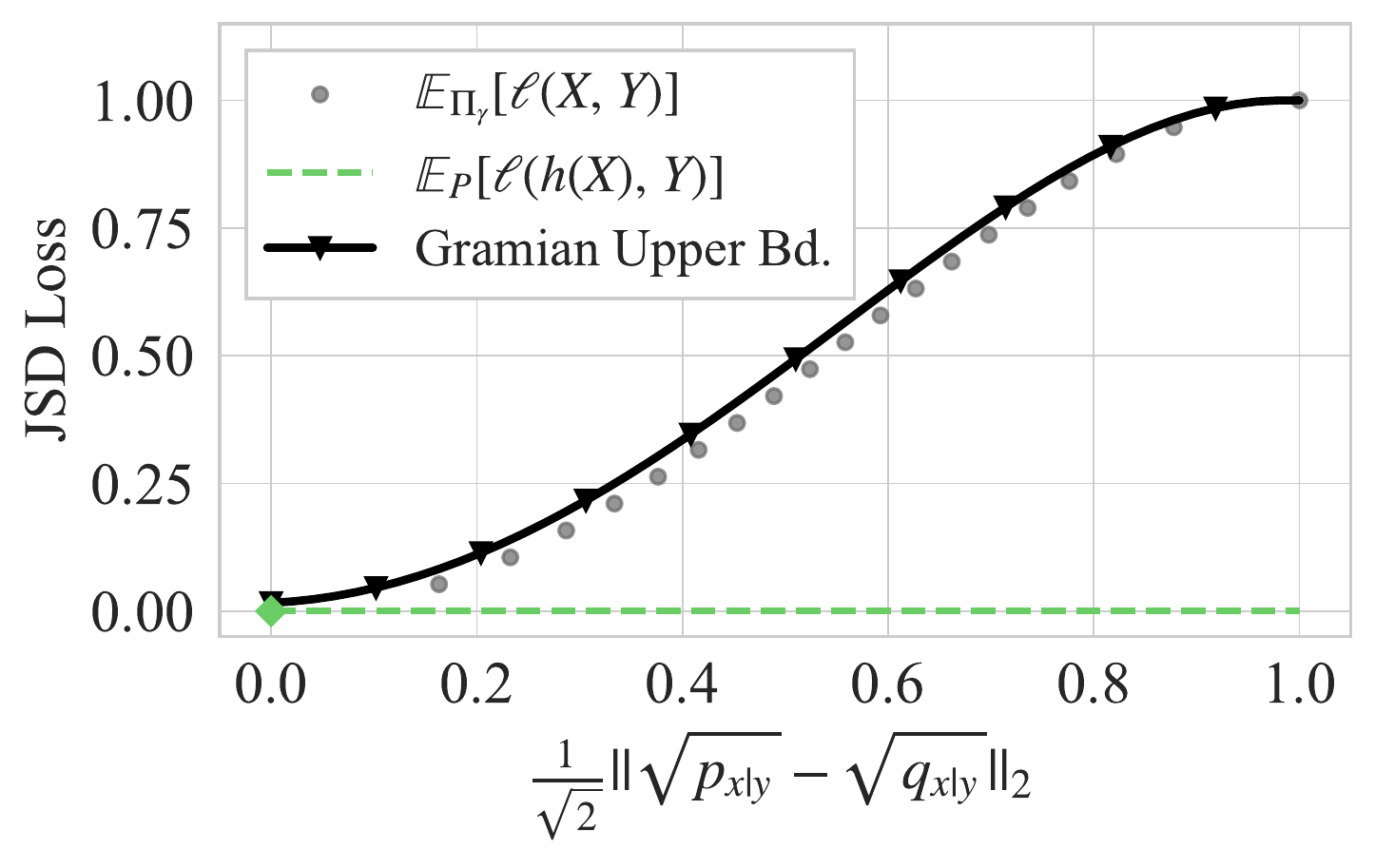}
\vspace{-2em}
\caption{Certified Jensen-Shannon divergence loss for the colored MNIST dataset.}
\vskip -0.2in
\end{figure}

\end{document}